\documentclass{article}
\usepackage{algorithm}
\usepackage{algpseudocode}
\usepackage{amssymb}

\usepackage[english]{babel}

\usepackage[letterpaper,top=2cm,bottom=2cm,left=3cm,right=3cm,marginparwidth=1.75cm]{geometry}

\usepackage{amsmath}
\usepackage{amsthm}

\usepackage{graphicx}
\usepackage{booktabs}
\usepackage{multirow}
\usepackage{tikz}
\usepackage{caption}
\usepackage{placeins}
\usetikzlibrary{matrix, positioning, calc, decorations.pathreplacing}
\usepackage[colorlinks=true, allcolors=blue]{hyperref}

\newtheorem{theorem}{Theorem}
\newtheorem{lemma}{Lemma}
\newtheorem{proposition}{Proposition}

\title{MXAttention: Data-Free Optimal Scaling and Pre-Normalization Quantization for MXFP4 Attention}

\author{
\parbox{\textwidth}{
\centering
Jianlin Yu, Jing Lin, Linghui Kong, Aiyue Chen, Weiyi Sun, Chenyu Zeng, Wangli Lan, Jinxi Li\\
Zhuo Zheng, Ziyang Yue, Danning Ke, Fei Yi, Tianchi Hu, Yuan Ding, Yiwu Yao, Junsong Wang\thanks{Corresponding author.}\\
Huawei Technologies Co., Ltd.
}
}

\begin{document}
\maketitle

\begin{abstract}
The quadratic cost of attention is a major bottleneck in diffusion-based video generation models. MXFP4 attention offers a promising path toward lower-cost inference, but directly applying standard MXFP4 quantization to attention often degrades generation quality. We attribute this degradation primarily to two numerical failure modes: power-of-two shared scaling creates a clipping--underflow trade-off within MXFP4 blocks, while direct MXFP4 quantization in the softmax loop breaks row-wise normalization, so the induced attention weights no longer sum to one after quantization. To address these issues, we propose \textbf{MXAttention}, a data-free post-training quantization framework for MXFP4 attention that achieves near-lossless generation quality through two key components. First, \textbf{Universal Optimal Scaling (UOS)} leverages the periodic structure induced by power-of-two microscaling to minimize a global MXFP4 quantization-error objective, deriving the closed-form, distribution-independent scaling boundary \(Q_{\max}=7.25\) without calibration or per-layer search. Second, \textbf{Pre-Normalization Quantization (PNQ)} quantizes unnormalized softmax exponentials before the row-wise summation, guaranteeing that the induced attention probabilities sum to one exactly and preventing row-sum errors from accumulating in the attention output. As a fully \textbf{data-free} method, MXAttention requires neither calibration nor QAT. Empirically, across Wan2.2 and HunyuanVideo, MXAttention closes at least 95\% of the VBench Imaging Quality gap between vanilla OCP MXFP4 and FP16, while substantially improving frame-level similarity. It preserves \textbf{FP16-level generation quality} with less than 0.01 absolute degradation on all reported VBench metrics and achieves performance competitive with strong NVFP4-based baselines with negligible overhead when fused into the attention pipeline. The implementation of MXAttention has been integrated into the main branch of MindIE-SD and is publicly available at
\url{https://gitcode.com/Ascend/MindIE-SD/tree/master/mindiesd}.
\end{abstract}

\section{Introduction}
\label{sec:intro}

Video diffusion Transformers process spatiotemporal token sequences whose length grows with spatial resolution and video duration~\cite{hunyuanvideo,wanvideo,cogvideox}. At every denoising step, the model performs a full forward pass through its Transformer blocks. Although fast samplers, distillation, and caching reduce the number of model evaluations or reuse intermediate computation~\cite{lcm,pcm,teacache,pab}, the cost of each executed forward pass remains substantial. Attention computation is therefore a key target for low-precision acceleration: its two dominant operations, \(QK^\top\) and \(PV\), are GEMMs that can directly exploit native FP4 throughput when \(Q\), \(K\), \(V\), and the softmax-path values can be quantized with sufficiently low error~\cite{zhang2025sageattention3,attn_qat}.

Native FP4 GEMM support and standardized microscaling formats have made 4-bit attention increasingly practical. MXFP4, defined by the OCP MX specification, is an open standard format supported across multiple accelerator families, including NVIDIA Blackwell GPUs, AMD Instinct MI350 series, and Ascend 950 series~\cite{ocp2023microscaling,nvidia_blackwell_mxfp4,amd_mi350_mxfp4,huawei_ascend950_mxfp4}. It represents E2M1 values in 32-element blocks using a shared E8M0 power-of-two scale. Alongside this open standard, NVIDIA introduced NVFP4, which adopts a finer 16-value block granularity with E4M3 block scales and an additional tensor-level scale~\cite{nvidia_nvfp4,transformer_engine_fp4}. MXFP4 reduces scale metadata overhead and enables simpler scale handling in low-precision GEMMs, but its coarser block granularity and power-of-two scaling restriction make accurate attention quantization challenging. Directly applying standard MXFP4 quantization to attention can therefore cause substantial degradation in generated video quality.

Recent work has explored low-bit attention from several directions. Rotation-based PTQ methods such as QuaRot and SpinQuant use orthogonal transformations, including Hadamard transforms, to suppress activation outliers and have become widely used in low-bit quantization pipelines~\cite{quarot,spinquant}. These techniques improve the numerical distributions of \(Q\) and \(K\), but do not directly address the format-specific scaling behavior of MXFP4, which can limit their effectiveness for MXFP4 attention. Attention-specific kernels such as the SageAttention series demonstrate the practicality of low-bit attention through smoothing, scaling, and kernel-level optimization~\cite{zhang2025sageattention,zhang2024sageattention2,zhang2025sageattention2plus,zhang2025sageattention3}. Quantization-aware training (QAT) further improves FP4 attention by adapting the model during training~\cite{attn_qat}. Recent MXFP4 quantization methods improve block quantization through overflow-aware scaling, adaptive scale selection, and metadata augmentation~\cite{oas,scalesearch,m2xfp}. These methods improve quantization accuracy through block-dependent decisions, empirical optimization, or additional format flexibility, whereas MXAttention derives a fixed, data-free MXFP4 scaling boundary. Video-diffusion quantization methods also exploit spatiotemporal structure to quantize model weights and activations, but generally do not address the numerical constraints inside the fused attention softmax loop~\cite{deltaquant}. MXAttention instead targets a fully data-free PTQ method for MXFP4 attention, without calibration, QAT, or per-layer search.

We characterize two numerical failure modes in directly quantized MXFP4 attention. First, power-of-two shared scaling creates a clipping--underflow trade-off within each MXFP4 block. A smaller shared scale preserves the resolution of smaller values but exposes larger values to saturation, whereas a larger shared scale reduces clipping at the cost of increased rounding and underflow. Second, directly inserting MXFP4 quantization into the FlashAttention online-softmax loop breaks row-wise normalization~\cite{dao2022flashattention,dao2023flashattention2,shah2024flashattention3}. In this setting, the unnormalized softmax exponential tile is quantized for the output-accumulator update, while the row-wise sum is accumulated from its unquantized counterpart. Because the two updates use different representations, the induced attention weights no longer necessarily sum to one.

To address these issues, we propose \textbf{MXAttention}, a data-free post-training quantization framework for MXFP4 attention. First, \textbf{Universal Optimal Scaling (UOS)} identifies a periodic structure induced by power-of-two microscaling and shows that it makes the optimal scaling boundary independent of the block-maximum distribution. Based on this observation, UOS derives the closed-form boundary \(Q_{\max}=7.25\) without calibration or per-layer search. Second, \textbf{Pre-Normalization Quantization (PNQ)} quantizes the unnormalized softmax exponentials before both the row-wise sum and output-accumulator updates. By using the same quantized values in both updates, PNQ preserves row-wise normalization by construction and avoids the additional scaling error caused by inconsistent quantization paths. Figure~\ref{fig:mxattention_workflow} provides an overview of MXAttention, which integrates Universal Optimal Scaling (UOS) and Pre-Normalization Quantization (PNQ) into the FlashAttention pipeline.

\begin{figure*}[t]
    \centering
    \includegraphics[width=\textwidth]{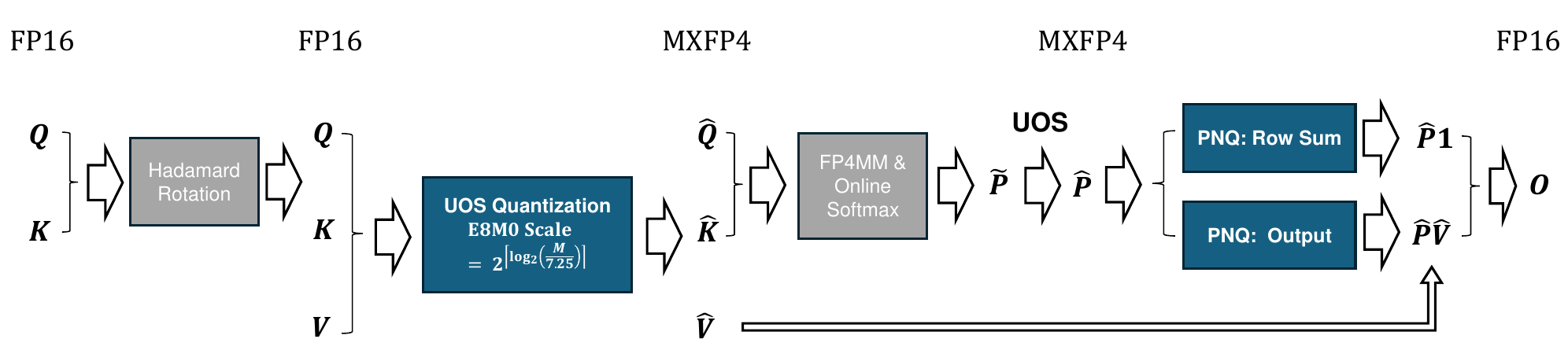}
    \caption{\textbf{Overview of MXAttention.} MXAttention integrates Universal Optimal Scaling (UOS) and Pre-Normalization Quantization (PNQ) into the FlashAttention pipeline. A fixed Hadamard rotation is applied to \(Q/K\) for outlier suppression. UOS applies the distribution-independent MXFP4 scaling boundary \(Q_{\max}=7.25\) to all MXFP4-quantized attention tensors. During the online-softmax loop, PNQ quantizes the unnormalized softmax tile and reuses the same quantized tile for both row-wise sum and output-accumulator updates, preserving row-wise normalization by construction.}
    \label{fig:mxattention_workflow}
\end{figure*}

Our main contributions are:
\begin{itemize}
    \item We characterize two numerical failure modes that arise when MXFP4 quantization is directly applied to attention: the clipping--underflow trade-off induced by power-of-two block scaling and the row-wise normalization error caused by inconsistent quantization in the online-softmax loop.

    \item We identify a periodic structure induced by power-of-two microscaling and show that it makes the optimal MXFP4 scaling boundary independent of the block-maximum distribution. Based on this observation, \textbf{Universal Optimal Scaling (UOS)} derives a single fixed boundary \(Q_{\max}=7.25\) for selecting the shared power-of-two scale of each MXFP4 block and proves that it globally minimizes the proposed quantization-error objective without calibration or per-layer search.

    \item We introduce \textbf{Pre-Normalization Quantization (PNQ)}, which uses the same quantized softmax exponential tiles for both the row-wise sum and output-accumulator updates, preserving row-wise normalization by construction and avoiding the additional scaling error introduced by mixing quantized and unquantized paths.

    \item We evaluate MXAttention on Wan2.2 and HunyuanVideo. Across both models, MXAttention closes at least 95\% of the VBench Imaging Quality gap between vanilla OCP MXFP4 and FP16, substantially improves frame-level similarity, and preserves FP16-level generation quality, remaining within 0.01 of or exceeding FP16 on every reported VBench metric. It also achieves performance competitive with strong NVFP4-based baselines with negligible algorithmic overhead and a fusion-friendly design.
\end{itemize}

\section{Preliminaries}
\label{sec:preliminaries}

\paragraph{Attention and FlashAttention.}
Given query, key, and value matrices \(Q\), \(K\), and \(V\), scaled dot-product attention computes
\begin{equation}
S=\frac{QK^\top}{\sqrt{d}}, \qquad P=\operatorname{Softmax}(S), \qquad O=PV,
\label{eq:standard_attention}
\end{equation}
where every row of \(P\) satisfies \(\sum_j P_{ij}=1\). A conventional implementation materializes the score matrix \(S\) and probability matrix \(P\), whose sizes grow quadratically with sequence length. FlashAttention avoids writing these matrices to HBM by partitioning \(Q\), \(K\), and \(V\) into blocks and evaluating attention through an online-softmax recurrence in on-chip memory~\cite{dao2022flashattention,dao2023flashattention2,shah2024flashattention3}.

For the recurrence below, we use the FlashAttention-2 form, which maintains an unnormalized output accumulator and applies the final normalization after all key/value blocks have been processed. We reserve \(P\) for the final normalized attention matrix, denote the unnormalized softmax exponential tile by \(\widetilde P_{ij}\), and denote the in-loop unnormalized output accumulator by \(\widetilde O_i^{(j)}\).

For query block \(Q_i\) and key/value blocks \(K_j,V_j\), define
\begin{equation}
S_{ij}=\frac{Q_iK_j^\top}{\sqrt{d}}.
\end{equation}
FlashAttention maintains a row-wise running maximum \(m_i^{(j)}\), a row-wise sum of exponentials \(\ell_i^{(j)}\), and an unnormalized output accumulator \(\widetilde O_i^{(j)}\). Starting from \(m_i^{(0)}=-\infty\), \(\ell_i^{(0)}=0\), and \(\widetilde O_i^{(0)}=0\), it updates
\begin{align}
m_i^{(j)} &= \max\!\left(m_i^{(j-1)},\operatorname{rowmax}(S_{ij})\right), \\
\alpha_i^{(j)} &= \exp\!\left(m_i^{(j-1)}-m_i^{(j)}\right), \\
\widetilde P_{ij} &= \exp\!\left(S_{ij}-m_i^{(j)}\mathbf{1}_{B_c}^{\top}\right), \\
\ell_i^{(j)} &= \alpha_i^{(j)}\odot \ell_i^{(j-1)}+\widetilde P_{ij}\mathbf{1}_{B_c}, \\
\widetilde O_i^{(j)} &= \operatorname{Diag}\!\left(\alpha_i^{(j)}\right)\widetilde O_i^{(j-1)}+\widetilde P_{ij}V_j.
\label{eq:online_attention}
\end{align}
Here, \(\widetilde P_{ij}\) contains the unnormalized softmax exponentials and enters the second attention GEMM; it is not the final normalized probability tile. After all \(T_c\) key/value blocks have been processed,
\begin{equation}
O_i=\operatorname{Diag}\!\left(\ell_i^{(T_c)}\right)^{-1}\widetilde O_i^{(T_c)}.
\label{eq:online_attention_output}
\end{equation}
FlashAttention therefore computes the normalized output from the row-wise sum \(\ell_i^{(T_c)}\) and the unnormalized output accumulator \(\widetilde O_i^{(T_c)}\), without materializing the full probability matrix \(P\). Materializing \(P\) for quantization would reintroduce quadratic intermediate storage and memory traffic. A FlashAttention-compatible low-bit implementation must instead integrate quantization into the tiled online-softmax loop.

\paragraph{MXFP4 Block Quantization.}
MXFP4 is a microscaling format defined by the OCP MX specification~\cite{ocp2023microscaling,rouhani2023microscaling}. Each block contains \(B=32\) E2M1 elements that share an E8M0 power-of-two scale \(\sigma=2^e\). The E2M1 grid is
\begin{equation}
\mathcal{G}_{\mathrm{E2M1}}=\{0,\pm0.5,\pm1,\pm1.5,\pm2,\pm3,\pm4,\pm6\}.
\end{equation}
For an element \(v\), we write MXFP4 quantization as
\begin{equation}
\mathcal{Q}_{\sigma}(v)=\sigma\,\Pi_{\mathcal{G}_{\mathrm{E2M1}}}\!\left(\frac{v}{\sigma}\right),
\label{eq:mxfp4_quantizer}
\end{equation}
where \(\Pi_{\mathcal{G}_{\mathrm{E2M1}}}\) denotes round-to-nearest projection onto the finite E2M1 grid, with out-of-range values saturated at \(\pm6\).

Let \(M=\max_{v\in\mathcal{B}}|v|\) be the maximum magnitude of a nonzero block \(\mathcal{B}\). The standard MX conversion rule computes
\begin{equation}
e_{\mathrm{OCP}}=\left\lfloor\log_2 M\right\rfloor-e_{\max}^{\mathrm{elem}}, \qquad \sigma_{\mathrm{OCP}}=2^{e_{\mathrm{OCP}}},
\label{eq:ocp_scale}
\end{equation}
where \(e_{\max}^{\mathrm{elem}}\) is the exponent of the largest power-of-two value in the element format~\cite{rouhani2023microscaling}. For E2M1, \(e_{\max}^{\mathrm{elem}}=2\), corresponding to the grid value \(4\).

We express different scale-selection rules through a common boundary \(Q_{\max}\):
\begin{align}
e_{\mathrm{floor}}(Q_{\max}) &= \left\lfloor\log_2\left(\frac{M}{Q_{\max}}\right)\right\rfloor+1, &
\sigma_{\mathrm{floor}}(Q_{\max}) &= 2^{e_{\mathrm{floor}}(Q_{\max})}, \\
e_{\mathrm{ceil}}(Q_{\max}) &= \left\lceil\log_2\left(\frac{M}{Q_{\max}}\right)\right\rceil, &
\sigma_{\mathrm{ceil}}(Q_{\max}) &= 2^{e_{\mathrm{ceil}}(Q_{\max})}.
\label{eq:scale_definition}
\end{align}
The exact OCP rule used in our experiments is \(e_{\mathrm{floor}}(8)\), which maps the normalized block maximum into \([4,8)\). Because \(8\) is not representable in E2M1, normalized block maxima above \(7\) fall into an overflow-rounding region and are saturated to \(6\).

TetraJet's Truncation-Free Scaling (TFS) corresponds to \(e_{\mathrm{ceil}}(6)\), giving
\begin{equation}
3<\frac{M}{\sigma_{\mathrm{ceil}}(6)}\le6,
\end{equation}
which avoids overflow of the block maximum but can increase rounding and underflow errors for smaller values~\cite{tetrajet}. The floor- and ceiling-based forms agree except when \(M/Q_{\max}\) is an exact power of two; this measure-zero endpoint distinction does not affect the continuous analysis in Section~\ref{subsec:uos}.

Our UOS method adopts the ceiling-based scaling rule and derives a fixed boundary \(Q_{\max}=7.25\) by minimizing a global MXFP4 quantization-error objective that captures the trade-off between clipping large values and preserving smaller values. It requires neither calibration nor per-layer search. For an all-zero block, we set \(\sigma=1\) and quantize every element to zero.

\section{Failure Modes of Standard MXFP4 Attention}
\label{sec:challenges}

Directly applying standard MXFP4 quantization to the attention pipeline can substantially degrade generated video quality. We trace this degradation to two distinct numerical failure modes: the clipping--underflow trade-off induced by block scaling and the normalization mismatch introduced in the online-softmax loop.

\paragraph{Clipping--Underflow Trade-off.}

Under the standard OCP conversion rule, the normalized block maximum
\(x_{\max}=M/\sigma_{\mathrm{OCP}}\) lies in the interval \([4,8)\)
~\cite{ocp2023microscaling,rouhani2023microscaling}. However, the largest finite E2M1 value is \(6\). Under round-to-nearest conversion followed by finite-range saturation, normalized maxima in \((7,8)\) fall into the overflow-rounding region: their unconstrained rounded value is \(8\), which is not representable in E2M1, and they are therefore saturated to \(6\).

Under the uniform wrapped-phase approximation discussed in Appendix~\ref{app:phase_uniformity}, \(x_{\max}\) has density

\[
p(x)=\frac{1}{x\ln 2}
\]

over \([4,8)\). The probability that the block maximum enters this region is therefore

\begin{equation}
\Pr(x_{\max}\ge 7)
=
\int_{7}^{8}\frac{1}{x\ln 2}\,dx
=
\log_2\frac{8}{7}
=
1-\log_2\frac{7}{4}
\approx 19.27\%.
\label{eq:ocp_clipping_probability}
\end{equation}

Truncation-Free Scaling (TFS) avoids this overflow by using
\(Q_{\max}=6\)~\cite{tetrajet}. When the OCP-normalized maximum exceeds \(6\), TFS doubles the shared scale and consequently halves the normalized magnitude of every value in the block.

Although this prevents saturation of the largest values, it reduces the effective resolution available to smaller values and increases their probability of rounding to zero. OCP scaling and TFS therefore represent opposite ends of the same trade-off: the former better preserves small values but permits overflow saturation, whereas the latter removes saturation at the cost of increased rounding and underflow.

\paragraph{Normalization Mismatch in Online Softmax.}
FlashAttention computes attention through a tiled online-softmax recurrence without materializing the full probability matrix~\cite{dao2022flashattention,dao2023flashattention2,shah2024flashattention3}. In the direct MXFP4 baseline considered in this work, the unnormalized exponential tile is quantized for the output-accumulator update, while its unquantized counterpart is used for the row-wise sum update. Let
\(\widehat P_{ij}=\mathcal{Q}_{\star}(\widetilde P_{ij})\)
denote the quantized exponential tile. The two update paths become
\begin{align}
\ell_i^{(j)}
&=
\alpha_i^{(j)}\odot \ell_i^{(j-1)}
+
\widetilde P_{ij}\mathbf{1}_{B_c},
\\
\widehat{\widetilde O}_i^{(j)}
&=
\operatorname{Diag}\!\left(\alpha_i^{(j)}\right)
\widehat{\widetilde O}_i^{(j-1)}
+
\widehat P_{ij}\widehat V_j .
\label{eq:naive_mxfp4_online_softmax}
\end{align}
The row-wise sum and output accumulator are therefore updated from different representations of the same softmax exponentials. Let \(\widetilde P_i^{\mathrm{eff}}\) and \(\widehat P_i^{\mathrm{eff}}\) denote the corresponding effective exponential rows after accounting for online rescaling factors. The induced attention weights satisfy
\begin{equation}
\widehat P_i^{\mathrm{direct}}
=
\operatorname{Diag}\!\left(
\widetilde P_i^{\mathrm{eff}}\mathbf{1}
\right)^{-1}
\widehat P_i^{\mathrm{eff}},
\qquad
\widehat P_i^{\mathrm{direct}}\mathbf{1}
=
\left(
\widehat P_i^{\mathrm{eff}}\mathbf{1}
\right)
\oslash
\left(
\widetilde P_i^{\mathrm{eff}}\mathbf{1}
\right)
\neq1 .
\label{eq:naive_rowsum}
\end{equation}
This mismatch is not an unavoidable consequence of element-wise quantization; it arises because the row-wise sum and output-accumulator updates use inconsistent representations of the same softmax exponentials. Zeroing, saturation, and rounding all perturb the quantized exponential mass. As quantified in Section~\ref{subsubsec:pnq_validation}, the induced row sums of the direct online MXFP4 baseline on Wan2.2 have an overall mean of \(0.9336\) across five denoising steps, two prompts, and 40 attention layers. All ten step--prompt groups have mean row sums below one, although individual rows range from \(0.7181\) to \(1.1127\). The resulting row-dependent scaling error perturbs the magnitude of the attention output and propagates through subsequent Transformer blocks and denoising steps.

\section{MXAttention}
\label{sec:mxattention}

This section presents \textbf{MXAttention}, which addresses MXFP4 quantization errors at two complementary stages of the attention pipeline. First, \textbf{Universal Optimal Scaling (UOS)} derives the closed-form, data-free optimal scaling boundary \(Q_{\max}=7.25\) by minimizing a global MXFP4 quantization-error objective. Second, \textbf{Pre-Normalization Quantization (PNQ)} quantizes the unnormalized softmax exponentials before both normalizer accumulation and output-accumulator updates, ensuring that the induced attention probabilities remain row-normalized by construction. Together, UOS and PNQ reduce block-level quantization error and eliminate the normalization mismatch introduced by direct MXFP4 quantization in online softmax.

\subsection{Universal Optimal Scaling}
\label{subsec:uos}

Under the ceiling-based scaling rule introduced in the Preliminaries, let \(q=Q_{\max}\) and let \(M\) be the maximum magnitude of a nonzero block. The shared-scale exponent and normalized block maximum are
\begin{equation}
e_q(M)=\left\lceil\log_2\left(\frac{M}{q}\right)\right\rceil,
\qquad
X_q=\frac{M}{2^{e_q(M)}}\in\left(\frac{q}{2},q\right].
\label{eq:uos_mapping}
\end{equation}
We analyze nonzero blocks and assume that \(e_q(M)\) lies within the E8M0 exponent range. The cases \(M=0\), for which \(\log_2M\) is undefined, and exponents outside this range are handled by the MXFP4 conversion routine and do not affect the derivation of \(q\). For a fixed block maximum \(M\), increasing \(q\) can select a smaller shared scale. This increases all normalized magnitudes, reducing rounding and underflow for small values while moving the block maximum closer to the upper end of the E2M1 range. Decreasing \(q\) has the opposite effect. UOS formulates this trade-off as an optimization over \(q\).

\paragraph{Data-Free Quantization-Error Objective.}
Because the E2M1 grid is symmetric, it is sufficient to analyze nonnegative magnitudes. Let \(\Pi_{\mathcal G_+}\) denote round-to-nearest projection onto the nonnegative finite E2M1 grid
\begin{equation}
\mathcal G_+=\{0,0.5,1,1.5,2,3,4,6\},
\end{equation}
with saturation at \(6\)~\cite{ocp2023microscaling,rouhani2023microscaling}. We define the cumulative projection error and its normalized form as
\begin{equation}
E(x)=\int_0^x\left(v-\Pi_{\mathcal G_+}(v)\right)^2\,dv,
\qquad
D(x)=\frac{E(x)}{x^3}.
\label{eq:uos_local_distortion}
\end{equation}
Equivalently,
\begin{equation}
D(x)=\frac{1}{x^2}\left[\frac{1}{x}\int_0^x\left(v-\Pi_{\mathcal G_+}(v)\right)^2\,dv\right].
\end{equation}
Thus, \(D(x)\) is the mean squared projection error over \([0,x]\), divided by \(x^2\). The factor \(1/x^2\) removes the overall magnitude scale, making \(D(x)\) a dimensionless relative-error measure. It is a data-free analytical measure determined by the E2M1 grid rather than the empirical MSE of a particular block.

Let \(g_q(x)\) denote the density of \(X_q\). We define the global MXFP4 quantization-error objective as
\begin{equation}
\mathcal J(q)=\mathbb E[D(X_q)]
=\int_{q/2}^{q}D(x)g_q(x)\,dx.
\label{eq:uos_global_objective}
\end{equation}

\paragraph{Log-Periodic Boundary Relation.}
A direct optimization of Eq.~\eqref{eq:uos_global_objective} appears to require knowledge of the block-maximum distribution, which can vary across tensors, layers, and models. We next show that power-of-two scaling imposes an exact log-periodic relation that removes this distribution dependence from the derivative sign. Recall that \(M\) is the block maximum, and let \(U=\log_2M\) have an arbitrary absolutely continuous density \(f_U\). Appendix~\ref{app:uos_density} shows that, within the support of \(X_q\),
\begin{equation}
g_q(x)=g(x)=\frac{H(\log_2x)}{x\ln2},
\qquad
H(t)=\sum_{k\in\mathbb Z}f_U(t+k),
\label{eq:uos_wrapped_density}
\end{equation}
Here, \(H\) is the periodized density of \(U\), obtained by summing the integer shifts of \(f_U\)~\cite{mardia2000directional,bell2024wrapped}. It has period one, meaning \(H(t+1)=H(t)\). The parameter \(q\) selects the interval \((q/2,q]\), but does not otherwise change the density within that interval.

The periodized density has unit mean over one period. Under standard Fourier-convergence conditions~\cite{stein2003fourier}, it can be written as
\begin{equation}
H(t)=1+\sum_{n\ne0}c_ne^{i2\pi nt},
\label{eq:uos_harmonic_decomposition}
\end{equation}
where the constant term corresponds to a uniform wrapped phase, or a log-uniform density in the linear domain, and the nonzero modes describe distribution-specific harmonics. These harmonics are not assumed to be small.

The endpoints \(q/2\) and \(q\) differ by exactly one power-of-two scale interval. After taking \(\log_2\), they are one unit apart, exactly the period of \(H\):
\begin{equation}
\log_2(q/2)=\log_2q-1.
\end{equation}
Consequently, every harmonic takes the same value at the two endpoints:
\begin{equation}
e^{i2\pi n\log_2(q/2)}
=e^{i2\pi n(\log_2q-1)}
=e^{i2\pi n\log_2q},
\qquad n\in\mathbb Z.
\label{eq:uos_harmonic_boundary_matching}
\end{equation}

\begin{lemma}[Boundary Density Ratio]
\label{lemma:periodic_boundary_matching}
For almost every \(q>0\), the density of the normalized block maximum satisfies
\begin{equation}
g\left(\frac q2\right)=2g(q).
\label{eq:uos_boundary_density}
\end{equation}
\end{lemma}

\begin{proof}
For almost every \(q\), periodicity gives
\begin{equation}
H\left(\log_2\frac q2\right)=H(\log_2q-1)=H(\log_2q).
\end{equation}
Substituting this identity into Eq.~\eqref{eq:uos_wrapped_density} gives
\begin{equation}
g\left(\frac q2\right)=\frac{H(\log_2q)}{(q/2)\ln2}=2g(q).
\end{equation}
\end{proof}

Since \(D(x)g(x)\) is locally integrable, \(\mathcal J\) is absolutely continuous; Appendix~\ref{app:differentiation} justifies the differentiation below. Differentiating Eq.~\eqref{eq:uos_global_objective} with respect to \(q\) using the Leibniz integral rule and applying Lemma~\ref{lemma:periodic_boundary_matching} yields, for almost every \(q\),
\begin{align}
\mathcal J'(q)
&=D(q)g(q)-\frac12D\left(\frac q2\right)g\left(\frac q2\right)\\
&=g(q)\left[D(q)-D\left(\frac q2\right)\right].
\label{eq:uos_objective_derivative}
\end{align}
All dependence on the block-maximum distribution is contained in the nonnegative factor \(g(q)\). This factor can change the magnitude of the derivative and the shape of the objective, but it cannot reverse the sign determined by \(D(q)-D(q/2)\). The location of the optimum is therefore determined by the E2M1 grid rather than by the block-maximum distribution.

Wherever \(g(q)>0\), the condition \(\mathcal J'(q)=0\) requires
\begin{equation}
D(q)=D\left(\frac q2\right),
\label{eq:uos_boundary_condition}
\end{equation}
or equivalently,
\begin{equation}
E(q)=8E\left(\frac q2\right).
\label{eq:uos_absolute_balance}
\end{equation}
Appendix~\ref{app:phase_uniformity} provides the complete Fourier interpretation and additional wrapped-distribution analysis~\cite{mardia2000directional,bell2024wrapped}. Neither a Gaussian nor a log-uniform prior is required for the UOS result.

\paragraph{Closed-Form E2M1 Solution.}
Equation~\eqref{eq:uos_objective_derivative} reduces the optimization to the grid-dependent difference \(D(q)-D(q/2)\). We now evaluate this difference exactly for E2M1. We first evaluate it over \(q\in[6,8]\). The endpoints recover TFS at \(q=6\)~\cite{tetrajet} and the ceiling-based counterpart of OCP at \(q=8\), which explains the choice of interval. Define
\begin{equation}
\Delta(q)=E(q)-8E\left(\frac q2\right).
\end{equation}
Since
\begin{equation}
D(q)-D\left(\frac q2\right)=\frac{\Delta(q)}{q^3},
\label{eq:uos_delta_relation}
\end{equation}
the sign of \(\Delta(q)\) determines the direction of \(\mathcal J(q)\). Exact integration over the E2M1 decision intervals gives
\begin{equation}
\Delta(q)=
\begin{cases}
-\dfrac18, & 6\le q\le7,\\[6pt]
\dfrac{(4q-29)(4q-27)}8, & 7\le q\le8.
\end{cases}
\label{eq:uos_piecewise_delta}
\end{equation}
The first branch is negative. In the second branch, the factor \(4q-27\) vanishes at \(q=27/4=6.75\), outside its valid interval \([7,8]\). The only admissible sign change is therefore
\begin{equation}
q^\star=\frac{29}{4}=7.25.
\end{equation}
Within \([6,8]\), the objective is nonincreasing before \(q^\star\) and nondecreasing afterward; the inequalities are strict wherever \(g(q)>0\). Appendix~\ref{app:uos_global_optimality} extends this sign analysis to the complete domain \(q>0\).

\begin{theorem}[Global Optimality of UOS]
\label{thm:uos_optimality}
Consider nonzero blocks for which the required E8M0 shared-scale exponent is representable. Under the ceiling-based scale family in Eq.~\eqref{eq:uos_mapping}, for any absolutely continuous law of \(U=\log_2M\), the objective in Eq.~\eqref{eq:uos_global_objective} has a global minimizer
\begin{equation}
q^\star=\frac{29}{4},
\qquad
Q_{\max}^\star=q^\star=7.25.
\label{eq:uos_optimal_qmax}
\end{equation}
If the periodized density \(H\) is positive almost everywhere over one period, the minimizer is unique.
\end{theorem}

Figure~\ref{fig:uos_analysis} plots the UOS objective under a log-uniform density and several nonuniform wrapped densities. The curves have different shapes but share the same minimum at \(q=7.25\), illustrating that distribution-specific harmonics affect the objective without shifting its grid-determined optimum.

\begin{figure*}[!t]
\centering
\begin{minipage}[t]{0.48\textwidth}
\centering
\includegraphics[width=\linewidth]{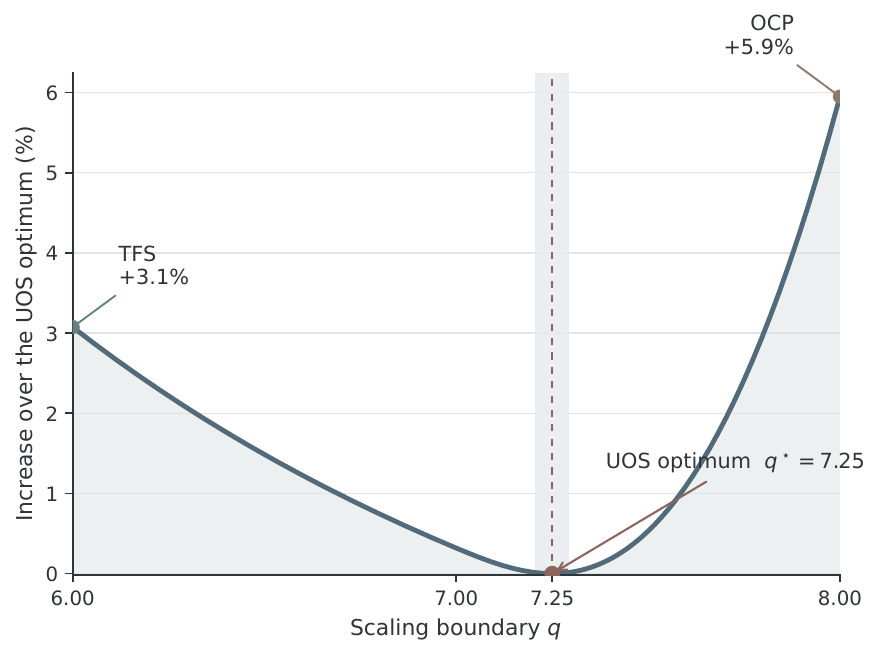}
\vspace{-0.35em}

{\small\textbf{(a) Objective under a log-uniform density}\par}
\end{minipage}
\hfill
\begin{minipage}[t]{0.48\textwidth}
\centering
\includegraphics[width=\linewidth]{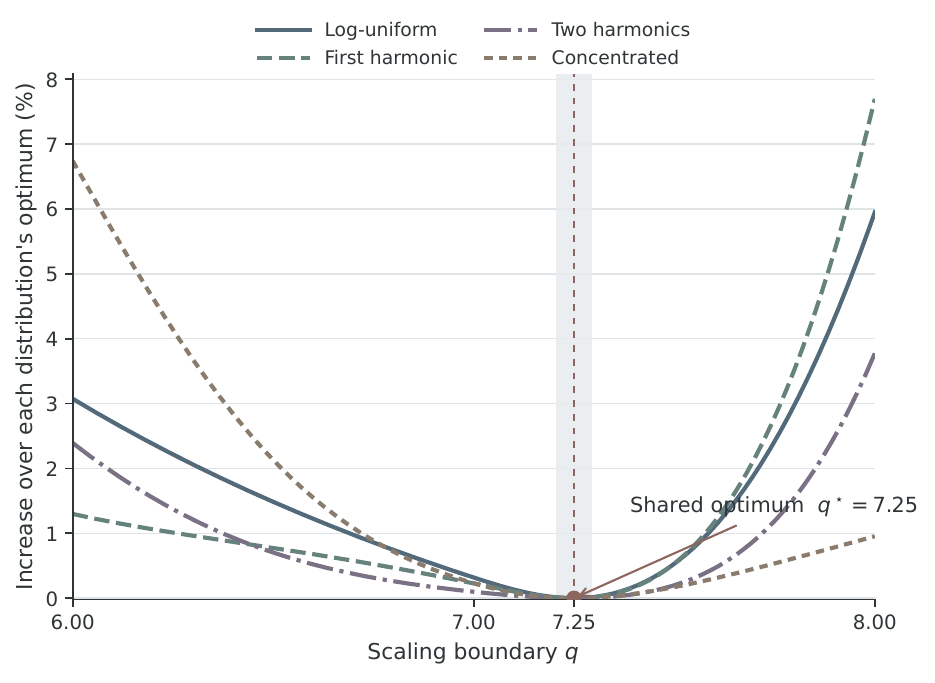}
\vspace{-0.35em}

{\small\textbf{(b) Objectives under wrapped densities}\par}
\end{minipage}
\caption{\textbf{UOS objective under different wrapped distributions.} (a) Percentage increase in the objective relative to its value at \(q=7.25\) under a log-uniform density. (b) The same quantity under representative nonuniform wrapped densities. Distribution-specific harmonics change the curve shapes, but all cases share the minimum \(q=7.25\).}
\label{fig:uos_analysis}
\end{figure*}

\paragraph{UOS Quantization.}
Algorithm~\ref{alg:uos_quantization} summarizes the resulting MXFP4 quantization procedure. UOS modifies only the shared-exponent selection rule, while block formation and element-wise E2M1 quantization remain unchanged.

\begin{algorithm}[t]
\caption{UOS-Based MXFP4 Block Quantization}
\label{alg:uos_quantization}
\begin{algorithmic}[1]
\Statex \textbf{Input:} A block
\(\mathcal B=\{v_i\}_{i=1}^{B}\), \(B=32\);
UOS boundary \(q^\star=7.25\)
\Statex \textbf{Output:} E8M0 shared scale \(\sigma\);
E2M1 elements \(\{p_i\}_{i=1}^{B}\)

\State \(M\gets\max_{1\le i\le B}|v_i|\)

\If{\(\mathcal B\text{ is an all-zero block}\)}
    \State \(\sigma\gets1\)
    \State \(p_i\gets0,\quad i=1,\ldots,B\)
    \State \textbf{return} \(\sigma,\{p_i\}_{i=1}^{B}\)
\EndIf

\State \(e\gets
\left\lceil
\log_2\!\left(M/q^\star\right)
\right\rceil\)

\State \(\sigma\gets2^e\)

\For{\(i=1\) to \(B\)}
    \State \(p_i\gets
    \Pi_{\mathcal G_{\mathrm{E2M1}}}
    \!\left(v_i/\sigma\right)\)
\EndFor

\State \textbf{return} \(\sigma,\{p_i\}_{i=1}^{B}\)
\end{algorithmic}
\end{algorithm}

For nonzero blocks, the normalized maximum after UOS scaling satisfies
\[
\frac{q^\star}{2}<\frac{M}{\sigma}\le q^\star,
\]
which yields the range \((3.625,7.25]\) for \(q^\star=7.25\).

\paragraph{Comparison with Existing Scaling Rules.}
Table~\ref{tab:uos_scaling_comparison} compares OCP, TFS, and UOS at the standard 32-element MXFP4 block granularity. The exact OCP baseline uses a floor-based conversion rule, whereas TFS and UOS belong to the ceiling-based family. The relation to OAS is given separately in Appendix~\ref{app:oas_relation}.

\begin{table}[!t]
\centering
\caption{\textbf{Comparison of MXFP4 scaling rules at \(B=32\).} All normalized maxima above \(6\) map to the largest finite E2M1 value. The reported overflow-rounding interval is the stricter subset in which the nearest value on an unbounded E2M1 ladder would be \(8\) before finite-range saturation. The endpoint at \(7\) depends on the tie-breaking convention.}
\label{tab:uos_scaling_comparison}
\small
\resizebox{\columnwidth}{!}{
\begin{tabular}{lccc}
\toprule
\textbf{Method} & \textbf{Shared-scale exponent} & \textbf{Normalized maximum range} & \textbf{Overflow-rounding interval}\\
\midrule
OCP & \(\lfloor\log_2M\rfloor-2\) & \([4,8)\) & \((7,8)\)\\
TFS & \(\lceil\log_2(M/6)\rceil\) & \((3,6]\) & None\\
UOS & \(\lceil\log_2(M/7.25)\rceil\) & \((3.625,7.25]\) & \((7,7.25]\)\\
\bottomrule
\end{tabular}}
\end{table}

Relative to TFS, UOS selects a smaller shared scale for a subset of blocks, improving the resolution available to smaller values while admitting only the narrow overflow-rounding interval \((7,7.25]\). Relative to OCP, UOS substantially narrows this interval. The optimum is not obtained by separately equating a clipping term and an underflow term. Instead, Eq.~\eqref{eq:uos_boundary_condition} balances the total normalized projection errors at the two endpoints associated with adjacent power-of-two scale choices.

\begin{figure}[!t]
\centering
\includegraphics[width=\columnwidth]{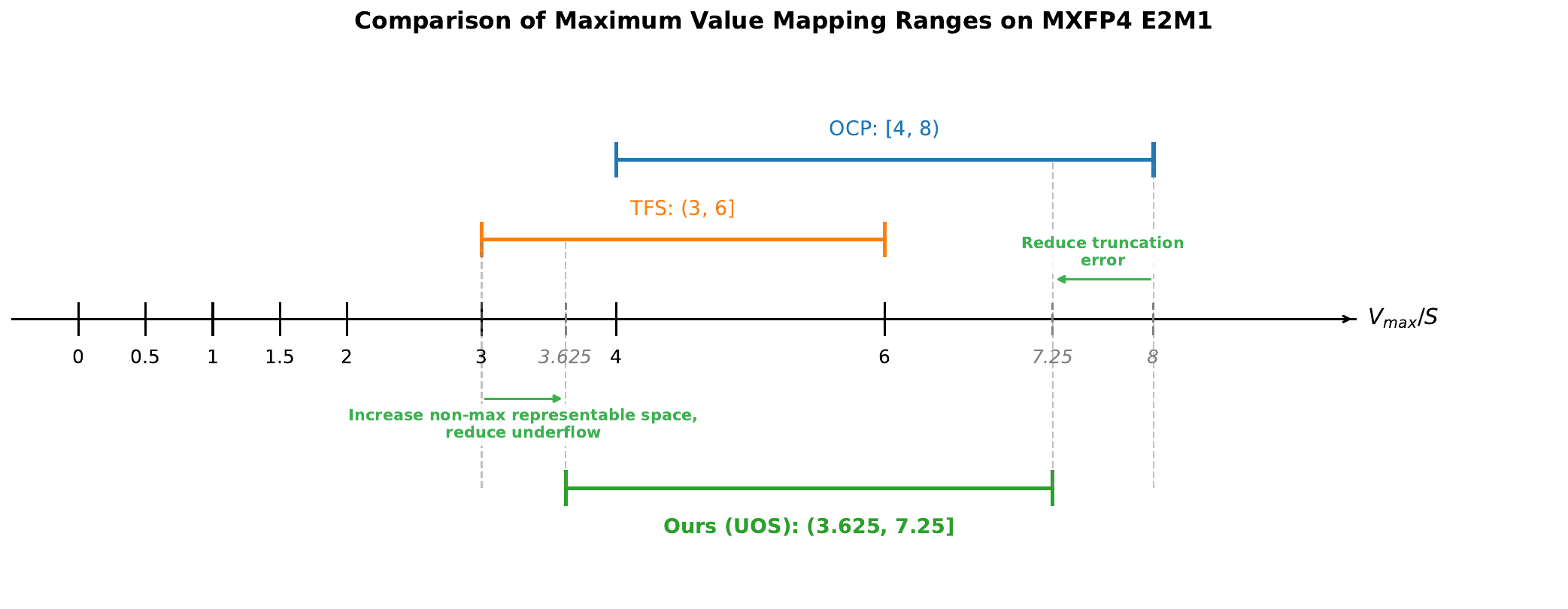}
\caption{\textbf{Normalized block-maximum ranges under representative MXFP4 scaling rules.} The exact OCP floor rule maps block maxima to \([4,8)\), with normalized maxima above \(7\) entering the overflow-rounding region. TFS uses \((3,6]\) to avoid overflow of the block maximum, while UOS selects the intermediate range \((3.625,7.25]\).}
\label{fig:mapping_ranges}
\end{figure}

\FloatBarrier

\subsection{Pre-Normalization Quantization}
\label{subsec:pnq}

Softmax normalizes each attention row:
\begin{equation}
\sum_j P_{ij}=1.
\label{eq:softmax_rowsum}
\end{equation}
Because MXFP4 quantizes microscaling blocks independently, it does not by itself preserve this row-wise sum. Consequently, where quantization is inserted in the online-softmax loop determines whether the induced attention weights remain normalized. Throughout this section, \(\mathcal{Q}_{\star}(\cdot)\) denotes block-wise MXFP4 quantization using the UOS boundary \(q^\star=7.25\). To isolate the effect of quantization placement, both the direct formulation and PNQ use the same quantizer \(\mathcal Q_\star\).

\begin{figure*}[!t]
\centering
\includegraphics[width=\textwidth]{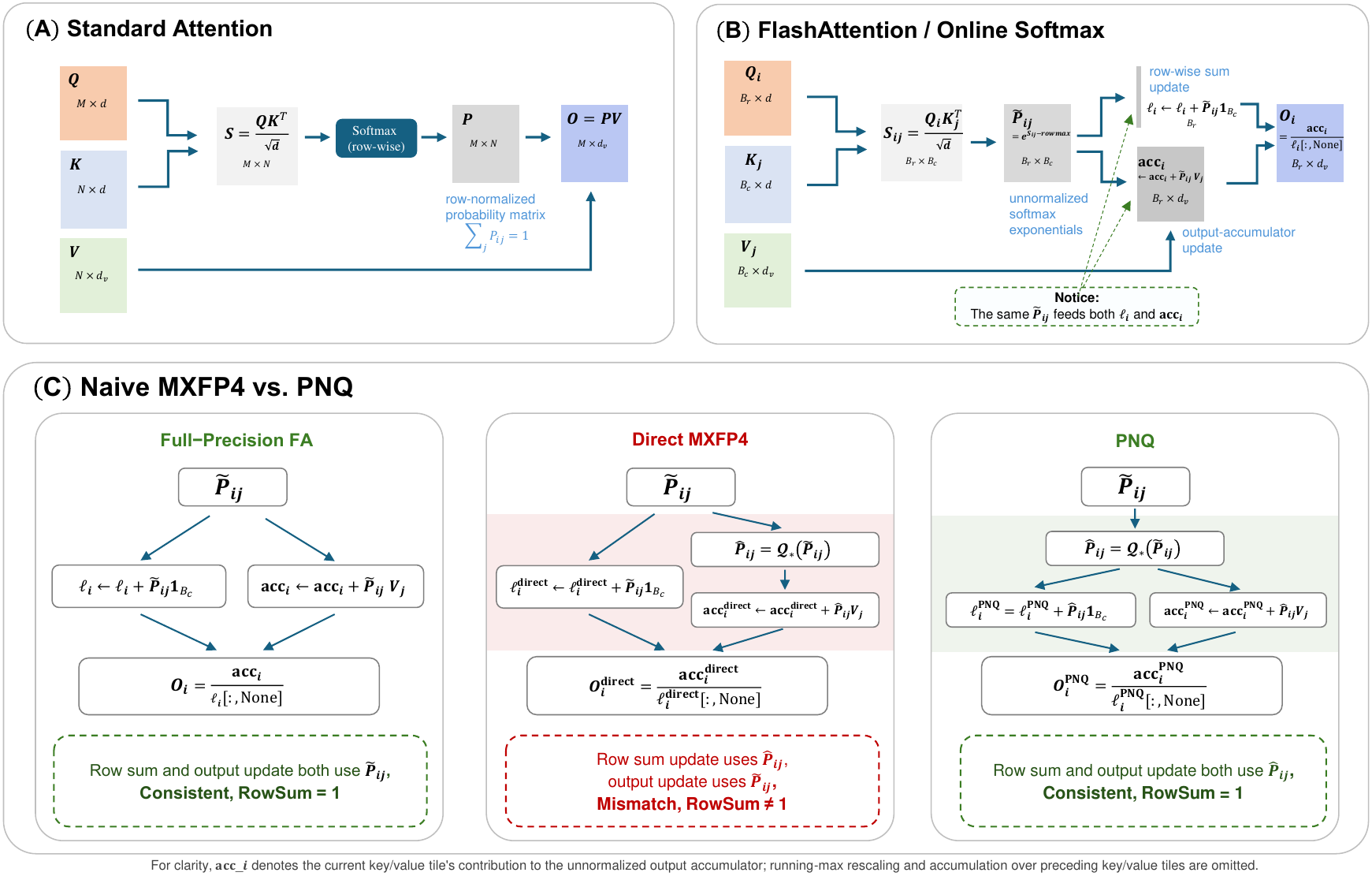}
\caption{\textbf{Pre-Normalization Quantization (PNQ).} (A) Standard attention forms the normalized probability matrix \(P\) before computing \(O=PV\). (B) FlashAttention avoids materializing \(P\) and updates both the row-wise sum of exponentials \(\ell_i\) and the unnormalized output accumulator \(\widetilde O_i\) from the same exponential tile \(\widetilde P_{ij}\). (C) A direct placement uses the unquantized tile \(\widetilde P_{ij}\) in the \(\ell_i\) update but the quantized tile \(\widehat P_{ij}\) in the \(\widetilde O_i\) update. PNQ uses the same \(\widehat P_{ij}\) in both updates. Running-max rescaling and previous key/value tiles are omitted for clarity.}
\label{fig:pnq_illustration}
\end{figure*}

Using the notation introduced in Section~\ref{sec:preliminaries}, consider a single tile before accounting for previous accumulators and running-max rescaling. The tile adds
\begin{equation}
\delta\ell_i=\widetilde P_{ij}\mathbf1_{B_c},
\qquad
\delta\widetilde O_i=\widetilde P_{ij}V_j
\label{eq:full_precision_tile_contribution}
\end{equation}
to the row-wise sum of exponentials and the unnormalized output accumulator. Both updates use the same exponential tile, as illustrated in Figure~\ref{fig:pnq_illustration}(B). For clarity, we omit the quantization mark on \(V_j\); its treatment is identical in the direct and PNQ formulations and does not affect the row-normalization analysis.

\paragraph{Normalization Mismatch under Direct Softmax-Path Quantization.}
A direct placement uses the unquantized exponential tile to update \(\ell_i\), but quantizes the same tile for the output-accumulator update:
\begin{equation}
\widehat P_{ij}=\mathcal Q_\star(\widetilde P_{ij}),
\qquad
\delta\ell_i^{\mathrm{direct}}=\widetilde P_{ij}\mathbf1_{B_c},
\qquad
\delta\widetilde O_i^{\mathrm{direct}}=\widehat P_{ij}V_j.
\label{eq:direct_mxfp4_tile}
\end{equation}
Here, \(\widehat P_{ij}\) denotes a quantized unnormalized exponential tile, not a normalized probability tile. The mismatch is already visible at the tile level because
\begin{equation}
\widehat P_{ij}\mathbf1_{B_c}
\neq
\widetilde P_{ij}\mathbf1_{B_c}
\label{eq:direct_tile_mass_mismatch}
\end{equation}
in general. The final row-wise sum is therefore computed from different weights than those used in the output update. Appendix~\ref{app:pnq_direct_full_row} gives the corresponding full-row expression after all online rescaling factors are included.

\paragraph{Pre-Normalization Quantization.}
PNQ quantizes the exponential tile before either update:
\begin{equation}
\widehat P_{ij}=\mathcal Q_\star(\widetilde P_{ij}),
\end{equation}
\begin{equation}
\delta\ell_i^{\mathrm{PNQ}}
=
\widehat P_{ij}\mathbf1_{B_c},
\qquad
\delta\widetilde O_i^{\mathrm{PNQ}}
=
\widehat P_{ij}V_j.
\label{eq:pnq_tile_contribution}
\end{equation}
The row-wise sum and output accumulator are therefore updated from the same quantized exponential values, as shown in Figure~\ref{fig:pnq_illustration}(C).

In the complete online-softmax recurrence, let
\begin{equation}
\alpha_i^{(j)}
=
\exp\!\left(m_i^{(j-1)}-m_i^{(j)}\right)
\end{equation}
be the row-wise factor applied when the running maximum changes. PNQ updates
\begin{align}
\ell_i^{(j)}
&=
\alpha_i^{(j)}\odot\ell_i^{(j-1)}
+
\widehat P_{ij}\mathbf1_{B_c},
\\
\widetilde O_i^{(j)}
&=
\operatorname{Diag}\!\left(\alpha_i^{(j)}\right)\widetilde O_i^{(j-1)}
+
\widehat P_{ij}V_j.
\label{eq:pnq_online_update}
\end{align}
Both FlashAttention states use the same running-max rescaling and the same quantized exponential tile.

Appendix~\ref{app:pnq_full_row} shows that, after all key/value tiles have been processed, the two states can be written as
\begin{equation}
\ell_i=\widehat P_i^{\mathrm{eff}}\mathbf1,
\qquad
\widetilde O_i=\widehat P_i^{\mathrm{eff}}V,
\qquad
O_i=\operatorname{Diag}(\ell_i)^{-1}\widetilde O_i,
\label{eq:pnq_final_output}
\end{equation}
where \(\widehat P_i^{\mathrm{eff}}\) is a conceptual full-row matrix obtained by rescaling each quantized tile to the final row maximum and concatenating the resulting tiles. It is used only for analysis and is not materialized by the kernel. The normalized attention weights induced by PNQ are
\begin{equation}
P_i^{\mathrm{PNQ}}
=
\operatorname{Diag}\!\left(\widehat P_i^{\mathrm{eff}}\mathbf1\right)^{-1}
\widehat P_i^{\mathrm{eff}}.
\label{eq:pnq_induced_weights}
\end{equation}

\begin{proposition}[Row-Wise Normalization under PNQ]
\label{prop:pnq_normalization}
For every attention row considered in this work, the weights induced by PNQ satisfy, in exact arithmetic,
\begin{equation}
P_i^{\mathrm{PNQ}}\mathbf1=\mathbf1.
\label{eq:pnq_rowsum}
\end{equation}
\end{proposition}

\begin{proof}
Substituting Eq.~\eqref{eq:pnq_induced_weights} gives
\begin{equation}
P_i^{\mathrm{PNQ}}\mathbf1
=
\operatorname{Diag}\!\left(\widehat P_i^{\mathrm{eff}}\mathbf1\right)^{-1}
\widehat P_i^{\mathrm{eff}}\mathbf1
=
\mathbf1.
\end{equation}
\end{proof}

PNQ does not change the element-wise MXFP4 quantizer. Its role is to make the resulting quantization error enter the \(\ell_i\) and \(\widetilde O_i\) updates consistently, preventing it from becoming an additional row-dependent scale error. For the attention rows considered in this work, the key set is nonempty. Appendix~\ref{app:pnq_positive_normalizer} shows that the tile containing the final row maximum contributes an exponential value of \(1\), which is represented exactly under \(q^\star=7.25\); hence, \(\ell_i>0\). Ordinary floating-point accumulation and division may still introduce their usual rounding error; Proposition~\ref{prop:pnq_normalization} concerns the normalization identity of the induced weights.

PNQ requires no additional pass over the attention matrix. The quantized exponential tile already used by the low-precision \(PV\) GEMM is reused in the \(\ell_i\) update. The mechanism-validation experiments evaluate the row-sum behavior of the direct placement and PNQ.

\FloatBarrier

\subsection{Overall MXAttention Pipeline}
\label{subsec:overall_pipeline}

MXAttention integrates UOS and PNQ into tiled attention computation. At runtime, a fixed orthogonal Hadamard rotation \(R_{\mathrm H}\) is applied to \(Q\) and \(K\) for outlier suppression. Let \(\mathcal{Q}_{\star}(\cdot)\) denote block-wise MXFP4 quantization using the fixed UOS boundary \(q^\star=7.25\). The MXFP4 quantization of \(Q\), \(K\), \(V\), and the unnormalized softmax exponential tiles \(\widetilde P_{ij}\) all uses this boundary. PNQ reuses each quantized tile \(\widehat P_{ij}\) in both the row-wise sum and output-accumulator updates.

\begin{center}
\begin{minipage}{0.98\columnwidth}
\hrule height 0.8pt
\vspace{0.35em}

\captionsetup{type=algorithm,skip=3pt}
\captionof{algorithm}{MXAttention Forward Pass}
\label{alg:mxattention}

\vspace{-0.25em}
\hrule height 0.4pt
\vspace{0.45em}

\begin{algorithmic}[1]
\Statex \textbf{Input:} Query, key, and value tensors \(Q,K,V\); fixed Hadamard rotation \(R_{\mathrm H}\); UOS boundary \(q^\star=7.25\)
\Statex \textbf{Output:} Attention output \(O\)

\State \(Q^{\mathrm{rot}}\gets QR_{\mathrm H}\)
\State \(K^{\mathrm{rot}}\gets KR_{\mathrm H}\)
\State \(\widehat Q\gets\mathcal{Q}_{\star}(Q^{\mathrm{rot}})\)
\State \(\widehat K\gets\mathcal{Q}_{\star}(K^{\mathrm{rot}})\)
\State \(\widehat V\gets\mathcal{Q}_{\star}(V)\)

\For{each query tile \(i\)}
    \State \(m_i\gets-\infty,\quad \ell_i\gets0,\quad \widetilde O_i\gets0\)
    \For{each key/value tile \(j\)}
        \State \(S_{ij}\gets\widehat Q_i\widehat K_j^\top/\sqrt d\)
        \State \(m_i^{\mathrm{new}}\gets\max\!\left(m_i,\operatorname{rowmax}(S_{ij})\right)\)
        \State \(\alpha_i\gets\exp\!\left(m_i-m_i^{\mathrm{new}}\right)\)
        \State \(\widetilde P_{ij}\gets\exp\!\left(S_{ij}-m_i^{\mathrm{new}}\mathbf1_{B_c}^{\top}\right)\)
        \State \(\widehat P_{ij}\gets\mathcal{Q}_{\star}(\widetilde P_{ij})\)
        \State \(\ell_i\gets\alpha_i\odot\ell_i+\widehat P_{ij}\mathbf1_{B_c}\)
        \State \(\widetilde O_i\gets\operatorname{Diag}(\alpha_i)\widetilde O_i+\widehat P_{ij}\widehat V_j\)
        \State \(m_i\gets m_i^{\mathrm{new}}\)
    \EndFor
    \State \(O_i\gets\operatorname{Diag}(\ell_i)^{-1}\widetilde O_i\)
\EndFor
\State \textbf{return} \(O\)
\end{algorithmic}

\vspace{0.35em}
\hrule height 0.8pt
\end{minipage}
\end{center}

\section{Experiments}
\label{sec:experiments}

We evaluate MXAttention on two large-scale video diffusion models through end-to-end generation quality, fully 4-bit ablations, and mechanism-level analyses of UOS and PNQ.

\subsection{Experimental Setup}

\noindent\textbf{Models.} We evaluate two large-scale open-source text-to-video models with distinct attention architectures: \texttt{Wan2.2-14B} and \texttt{HunyuanVideo-13B}. \texttt{Wan2.2} employs separate self-attention and cross-attention modules, whereas \texttt{HunyuanVideo} adopts dual-stream and single-stream Transformer blocks with full attention.

\vspace{0.3em}

\noindent\textbf{Datasets \& Generation.} We use the same fixed prompt subset from the Open-Sora prompt suite for all compared methods~\cite{zheng2024opensora}. For each model and prompt, all methods use the same generation seed, enabling paired comparisons with the FP16 output.

\vspace{0.3em}

\noindent\textbf{Evaluation Metrics.} We evaluate generated-video quality with VBench~\cite{vbench}, reporting \textit{Subject Consistency}, \textit{Imaging Quality}, and \textit{Aesthetic Quality}, averaged over all generated videos. To measure frame-level similarity to the FP16 baseline, we report frame-level cosine similarity, Structural Similarity Index Measure (SSIM), and Peak Signal-to-Noise Ratio (PSNR), averaged over frames.

\vspace{0.3em}

\noindent\textbf{Implementation Details.} We follow the recommended inference settings of both models. Wan2.2 generates 81-frame videos at 720p resolution using 40 denoising steps, while HunyuanVideo generates 129-frame videos at 720p resolution using 50 denoising steps. Only the Wan2.2 main comparison in Table~\ref{tab:main_vbench_results} uses a fixed hybrid-precision policy: Block~0 and the final two denoising steps (38 and 39) remain in high precision, while all remaining attention computations use 4-bit attention. All other experiments use fully 4-bit attention, including the HunyuanVideo main comparison, the ablation studies on both models, and all mechanism-validation experiments. Here, fully 4-bit attention means that every attention block at every denoising step uses the 4-bit attention kernel; non-attention model components retain their original precision.

\subsection{Baselines and Configurations}

We compare MXAttention with FP16 and representative 4-bit attention configurations:
\begin{itemize}
    \item \textbf{FP16.} The unmodified full-precision attention implementation.
    \item \textbf{MXFP4 (OCP).} A direct MXFP4 attention baseline using the exact floor-based OCP conversion rule~\cite{ocp2023microscaling,rouhani2023microscaling}. The unnormalized softmax tile \(\widetilde P\) is quantized into \(\widehat P=\mathcal{Q}_{\star}(\widetilde P)\) for the output-accumulator update, while the row-wise sum \(\ell\) is updated from the unquantized tile.
    \item \textbf{NVFP4.} A direct NVFP4 implementation in which \(Q\), \(K\), \(V\), and \(\widetilde P\) are quantized for the two attention GEMMs. No additional smoothing or softmax-path scaling is applied~\cite{nvidia_nvfp4,transformer_engine_fp4}.
    \item \textbf{NVFP4 + SageAttention.} The NVFP4 baseline augmented with channel-wise \(Q/K\) smoothing and two-level scaling for the softmax-path tile, following the SageAttention design~\cite{zhang2025sageattention3}.
    \item \textbf{MXAttention.} Our method applies a fixed Hadamard rotation to \(Q/K\) for outlier suppression~\cite{quarot,spinquant}, uses the UOS boundary \(Q_{\max}=7.25\) for all MXFP4 quantization steps, and applies PNQ so that the same quantized softmax tile updates both \(\ell\) and the output-accumulator state.
\end{itemize}

\subsection{Main Results}
\label{subsec:main_results}

\begin{table*}[!t]
\centering
\caption{\textbf{End-to-end evaluation of 4-bit attention methods.} All 4-bit Wan2.2 methods use the same hybrid-precision policy, retaining high precision for Block~0 and denoising steps 38--39. HunyuanVideo uses fully 4-bit attention. Bold values denote the best result among quantized methods.}
\label{tab:main_vbench_results}
\resizebox{\textwidth}{!}{
\begin{tabular}{llcccccc}
\toprule
\multirow{2}{*}{\textbf{Model}} &
\multirow{2}{*}{\textbf{Method}} &
\multicolumn{3}{c}{\textbf{VBench Metrics} (\(\uparrow\))} &
\multicolumn{3}{c}{\textbf{Similarity to FP16} (\(\uparrow\))} \\
\cmidrule(lr){3-5}\cmidrule(lr){6-8}
& & \textbf{Subject} & \textbf{Imaging} & \textbf{Aesthetic} & \textbf{Cosine} & \textbf{SSIM} & \textbf{PSNR} \\
\midrule
\multirow{5}{*}{Wan2.2}
& FP16 & 0.9562 & 0.7085 & 0.6042 & -- & -- & -- \\
\cmidrule(lr){2-8}
& MXFP4 (OCP) & 0.9517 & 0.6414 & 0.6202 & 0.9290 & 0.5076 & 15.58 \\
& NVFP4 & \textbf{0.9583} & 0.6973 & 0.6158 & 0.9278 & 0.5275 & 15.52 \\
& NVFP4 + SageAttention & 0.9558 & 0.6980 & 0.6127 & 0.9510 & 0.6177 & 17.72 \\
& MXAttention & 0.9544 & \textbf{0.7054} & \textbf{0.6229} & \textbf{0.9536} & \textbf{0.6319} & \textbf{17.92} \\
\midrule
\multirow{5}{*}{HunyuanVideo}
& FP16 & 0.9665 & 0.6185 & 0.6242 & -- & -- & -- \\
\cmidrule(lr){2-8}
& MXFP4 (OCP) & \textbf{0.9781} & 0.4459 & 0.5928 & 0.9489 & 0.5954 & 16.25 \\
& NVFP4 & 0.9666 & 0.6221 & 0.6162 & 0.9635 & 0.6359 & 17.80 \\
& NVFP4 + SageAttention & 0.9649 & 0.6063 & 0.6288 & \textbf{0.9800} & \textbf{0.7424} & \textbf{20.57} \\
& MXAttention & 0.9690 & \textbf{0.6380} & \textbf{0.6344} & 0.9745 & 0.7061 & 19.23 \\
\bottomrule
\end{tabular}
}
\end{table*}

Table~\ref{tab:main_vbench_results} shows that direct OCP MXFP4 attention causes a substantial loss in generation quality. Imaging Quality decreases from 0.7085 to 0.6414 on Wan2.2 and from 0.6185 to 0.4459 on HunyuanVideo. MXAttention raises the scores to 0.7054 and 0.6380, respectively: it recovers 95.4\% of the FP16--MXFP4 gap on Wan2.2 and fully closes the gap on HunyuanVideo, where it exceeds the FP16 score.

Among the quantized methods, MXAttention achieves the highest Imaging and Aesthetic scores on both models. On Wan2.2, it also obtains the highest cosine similarity, SSIM, and PSNR, leading five of the six reported metrics. Its Imaging Quality is only 0.0031 below FP16, while its Aesthetic Quality exceeds the FP16 baseline.

The HunyuanVideo comparison uses fully 4-bit attention across all blocks and denoising steps. Relative to direct OCP MXFP4, MXAttention improves Imaging Quality by 0.1921 and Aesthetic Quality by 0.0416, while increasing SSIM from 0.5954 to 0.7061 and PSNR from 16.25 to 19.23~dB. It also exceeds FP16 in both Imaging and Aesthetic Quality. NVFP4 + SageAttention obtains the highest paired frame-level similarity on this model, while MXAttention leads Imaging and Aesthetic Quality.

Overall, MXAttention recovers the generation-quality loss of direct MXFP4 and substantially improves similarity to FP16. Figure~\ref{fig:video_snapshot} provides qualitative comparisons under identical prompts and random seeds. Direct OCP MXFP4 produces noticeable changes in object identity and appearance compared with FP16. NVFP4-based methods reduce these artifacts but still exhibit minor appearance variations in some cases. MXAttention better preserves the semantic content and visual appearance of the FP16 generation, consistent with the improvements in Imaging Quality and frame-level similarity. These results demonstrate that MXFP4 attention with proper quantization techniques can retain FP16-level generation quality and remain competitive with strong NVFP4-based pipelines.

\begin{figure*}[!t]
\centering
\includegraphics[width=\textwidth]{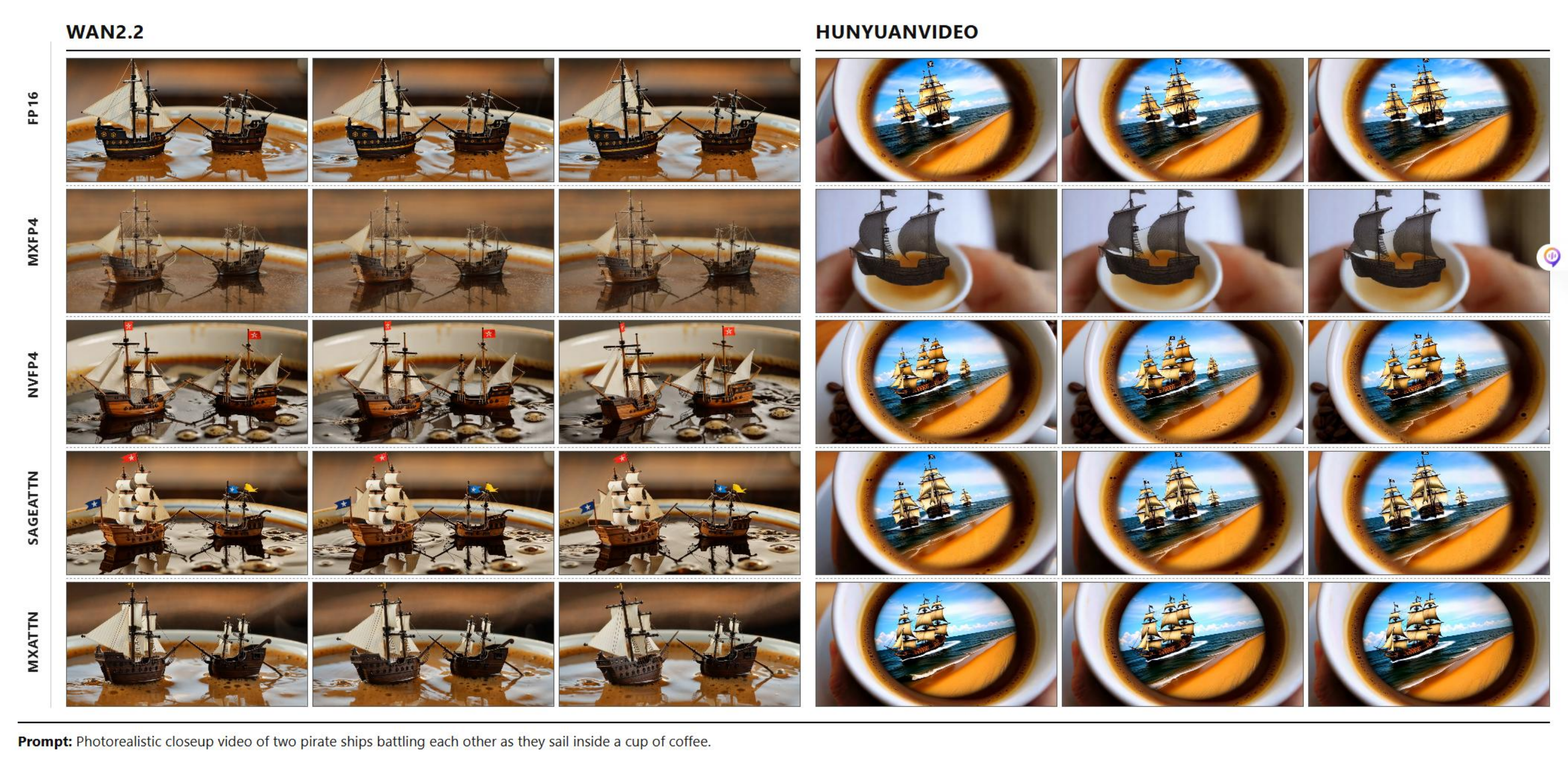}
\caption{\textbf{Qualitative comparison of 4-bit attention methods.} Each row shows frames generated with the same prompt and random seed. The left three columns correspond to a Wan2.2 video and the right three columns correspond to a HunyuanVideo video, with multiple frames sampled from the same generated video. Direct OCP MXFP4 produces noticeable changes in object identity and appearance, while MXAttention better preserves the FP16 generation.}
\label{fig:video_snapshot}
\end{figure*}

\FloatBarrier

\subsection{Ablation Studies}
\label{sec:ablation}

We evaluate the contribution of the fixed Hadamard rotation, PNQ, and UOS under fully 4-bit attention. Unlike the main Wan2.2 comparison, no high-precision block or denoising-step fallback is used in this study. The rows form a sequential ablation: UOS is first removed from the full method, followed by PNQ and then the fixed Hadamard rotation.

\begin{table*}[!t]
\centering
\caption{\textbf{Sequential ablation under fully 4-bit attention.} All attention blocks and denoising steps are quantized without high-precision fallback. Bold values denote the best result within each model.}
\label{tab:ablation_results}
\resizebox{\textwidth}{!}{
\begin{tabular}{llccc|cccccc}
\toprule
\multirow{2}{*}{\textbf{Model}} &
\multirow{2}{*}{\textbf{Configuration}} &
\multirow{2}{*}{\textbf{Hadamard}} &
\multirow{2}{*}{\textbf{PNQ}} &
\multirow{2}{*}{\textbf{UOS}} &
\multicolumn{3}{c}{\textbf{VBench Metrics} (\(\uparrow\))} &
\multicolumn{3}{c}{\textbf{Similarity to FP16} (\(\uparrow\))} \\
\cmidrule(lr){6-8}\cmidrule(lr){9-11}
& & & & & \textbf{Subject} & \textbf{Imaging} & \textbf{Aesthetic} & \textbf{Cosine} & \textbf{SSIM} & \textbf{PSNR} \\
\midrule
\multirow{4}{*}{Wan2.2}
& Full & \(\checkmark\) & \(\checkmark\) & \(\checkmark\) & 0.9424 & \textbf{0.6842} & 0.5994 & \textbf{0.9329} & \textbf{0.5519} & \textbf{15.96} \\
& w/o UOS & \(\checkmark\) & \(\checkmark\) & -- & 0.9402 & 0.6822 & \textbf{0.6176} & 0.9143 & 0.5084 & 14.90 \\
& w/o PNQ and UOS & \(\checkmark\) & -- & -- & 0.9442 & 0.6352 & 0.6081 & 0.9063 & 0.4727 & 14.77 \\
& OCP MXFP4 & -- & -- & -- & \textbf{0.9461} & 0.5452 & 0.5998 & 0.9127 & 0.4636 & 14.78 \\
\midrule
\multirow{4}{*}{HunyuanVideo}
& Full & \(\checkmark\) & \(\checkmark\) & \(\checkmark\) & 0.9689 & 0.6380 & \textbf{0.6344} & \textbf{0.9745} & \textbf{0.7061} & \textbf{19.23} \\
& w/o UOS & \(\checkmark\) & \(\checkmark\) & -- & 0.9713 & \textbf{0.6423} & 0.6229 & 0.9676 & 0.6740 & 18.41 \\
& w/o PNQ and UOS & \(\checkmark\) & -- & -- & 0.9759 & 0.5321 & 0.6046 & 0.9643 & 0.6520 & 17.87 \\
& OCP MXFP4 & -- & -- & -- & \textbf{0.9781} & 0.4459 & 0.5928 & 0.9489 & 0.5954 & 16.25 \\
\bottomrule
\end{tabular}
}
\end{table*}

Table~\ref{tab:ablation_results} shows that UOS, PNQ, and the fixed Hadamard rotation provide complementary gains under fully 4-bit attention.

UOS consistently improves cosine similarity, SSIM, and PSNR on both models. Compared with the configuration without UOS, the full method improves these metrics from 0.9143/0.5084/14.90 to 0.9329/0.5519/15.96 on Wan2.2 and from 0.9676/0.6740/18.41 to 0.9745/0.7061/19.23 on HunyuanVideo. The consistent gains in paired similarity support the analytical boundary \(Q_{\max}=7.25\) as a more accurate MXFP4 scale choice.

PNQ produces the largest gain in Imaging Quality. Adding PNQ increases the score from 0.6352 to 0.6822 on Wan2.2 and from 0.5321 to 0.6423 on HunyuanVideo, while also improving all three paired similarity metrics. These gains agree with the analysis in Section~\ref{subsec:pnq}: using the same quantized exponential tile for the row-wise sum and output-accumulator updates avoids the additional row-dependent scale error of direct softmax-path quantization.

The fixed Hadamard rotation further reduces the effect of \(Q/K\) outliers. Adding it to direct OCP MXFP4 raises Imaging Quality from 0.5452 to 0.6352 on Wan2.2 and from 0.4459 to 0.5321 on HunyuanVideo. On HunyuanVideo, it also improves cosine similarity, SSIM, and PSNR.

The full configuration combines the highest cosine similarity, SSIM, and PSNR on both models with the highest Imaging Quality on Wan2.2 and the highest Aesthetic Quality on HunyuanVideo. These results show that UOS, PNQ, and the fixed Hadamard rotation address complementary sources of MXFP4 error and together provide the strongest overall balance.

Figure~\ref{fig:video_snapshot2} provides qualitative ablation results under fully 4-bit attention. Removing UOS gradually introduces more visual artifacts, while further removing PNQ leads to additional degradation in generated details. Direct OCP MXFP4 exhibits the most severe quality degradation, with noticeable blurring and artifact patterns. The qualitative results are consistent with the quantitative ablation study, showing that UOS and PNQ provide complementary improvements.

\begin{figure*}[!t]
\centering
\includegraphics[width=\textwidth]{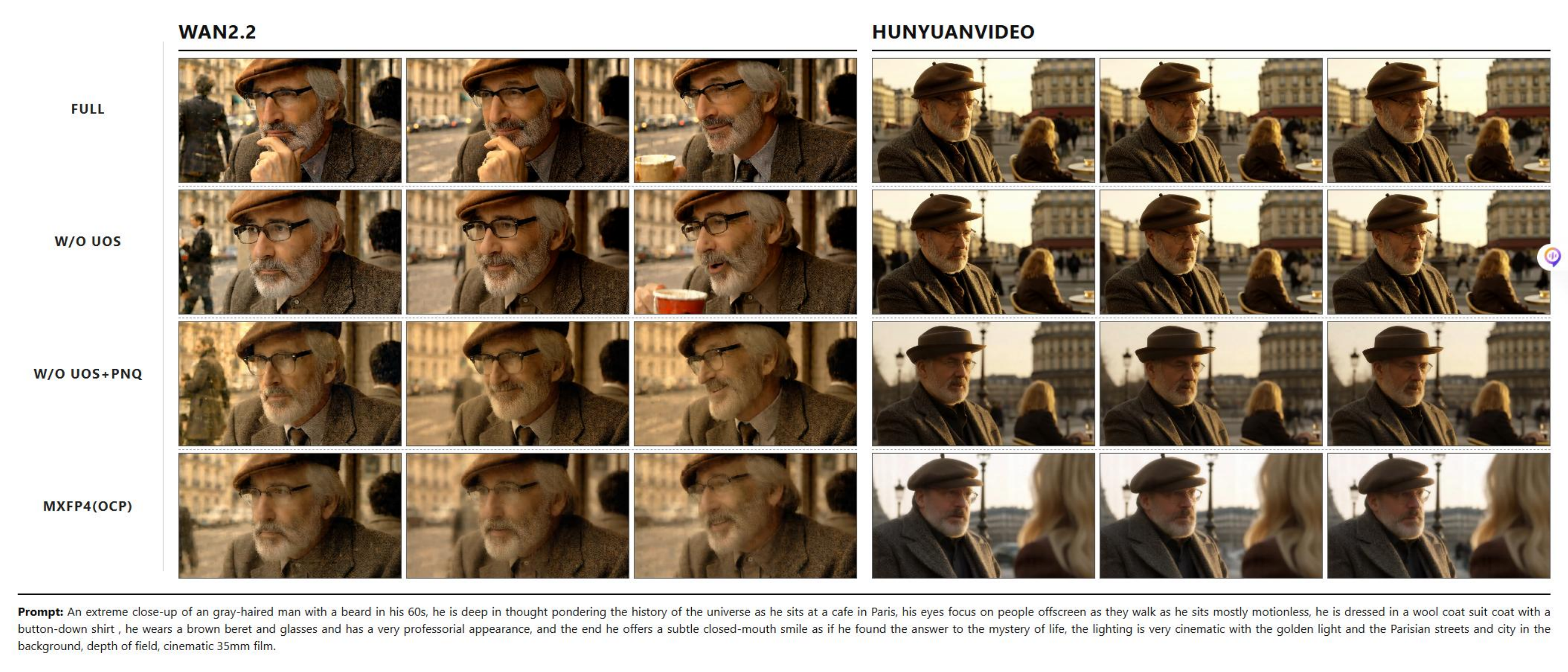}
\caption{\textbf{Qualitative ablation under fully 4-bit attention.} From top to bottom: full MXAttention, without UOS, without UOS and PNQ, and direct OCP MXFP4. Removing UOS and PNQ progressively increases visible artifacts and reduces generation quality.}
\label{fig:video_snapshot2}
\end{figure*}

\FloatBarrier

\subsection{Mechanism Validation}
\label{sec:mechanism_validation}

\subsubsection{Validation of Universal Optimal Scaling}
\label{subsubsec:uos_validation}

We compare UOS with the OCP boundary \(Q_{\max}=8\), TFS with \(Q_{\max}=6\), and a per-block empirical oracle obtained by sweeping \(Q_{\max}\) from 6.00 to 8.00 in increments of 0.05. At the initial denoising step, the oracle independently minimizes MSE for \(Q\), \(K\), and \(V\) in each of the 40 evaluated attention blocks.

\begin{table}[!t]
\centering
\caption{\textbf{Per-block empirical \(Q_{\max}\) optima.} Values are percentages over 40 attention blocks.}
\label{tab:cx_distribution}
\scriptsize
\setlength{\tabcolsep}{3.5pt}
\renewcommand{\arraystretch}{1.08}
\begin{tabular}{llrrrr}
\toprule
\textbf{Model} & \textbf{Tensor}
& \(\boldsymbol{\leq 7.15}\)
& \(\boldsymbol{7.20}\)
& \(\boldsymbol{7.25}\)
& \(\boldsymbol{\geq 7.30}\) \\
\midrule
\multirow{3}{*}{Wan2.2}
& \(Q\) & 0.0 & 5.0  & \textbf{95.0} & 0.0 \\
& \(K\) & 0.0 & 10.0 & \textbf{90.0} & 0.0 \\
& \(V\) & 0.0 & 10.0 & \textbf{90.0} & 0.0 \\
\midrule
\multirow{3}{*}{HunyuanVideo}
& \(Q\) & 45.0 & 0.0  & \textbf{52.5} & 2.5 \\
& \(K\) & 0.0  & \textbf{62.5} & 32.5 & 5.0 \\
& \(V\) & 0.0  & 50.0 & 50.0 & 0.0 \\
\bottomrule
\end{tabular}
\end{table}

For Wan2.2, \(Q_{\max}=7.25\) is the empirical optimum for 95.0\%, 90.0\%, and 90.0\% of the evaluated \(Q\), \(K\), and \(V\) blocks, respectively. On HunyuanVideo, \(7.25\) is the most frequent optimum for \(Q\) and is tied with \(7.20\) for \(V\), while the \(K\) optima concentrate at \(7.20\) (62.5\%) and \(7.25\) (32.5\%). Overall, \(7.25\) is the mode or co-mode in five of the six tensor--model combinations and the second-most frequent choice in the remaining case.

The sweep minimizes finite-sample tensor MSE, whereas UOS minimizes the data-free analytical objective in Eq.~\eqref{eq:uos_global_objective}. Their agreement supports using \(Q_{\max}=7.25\) without layer-wise calibration. The softmax path is validated separately below.

\subsubsection{Validation of Pre-Normalization Quantization}
\label{subsubsec:pnq_validation}

We use materialized attention probabilities as a diagnostic for the row-wise mass error introduced by block-wise MXFP4 quantization. FlashAttention does not materialize \(P\) during inference, and PNQ quantizes the unnormalized exponential tile \(\widetilde P\). This diagnostic isolates how local MXFP4 rounding changes row-wise probability mass and complements the online-softmax analysis in Section~\ref{subsec:pnq}.

\begin{figure*}[!t]
\centering
\includegraphics[width=\textwidth]{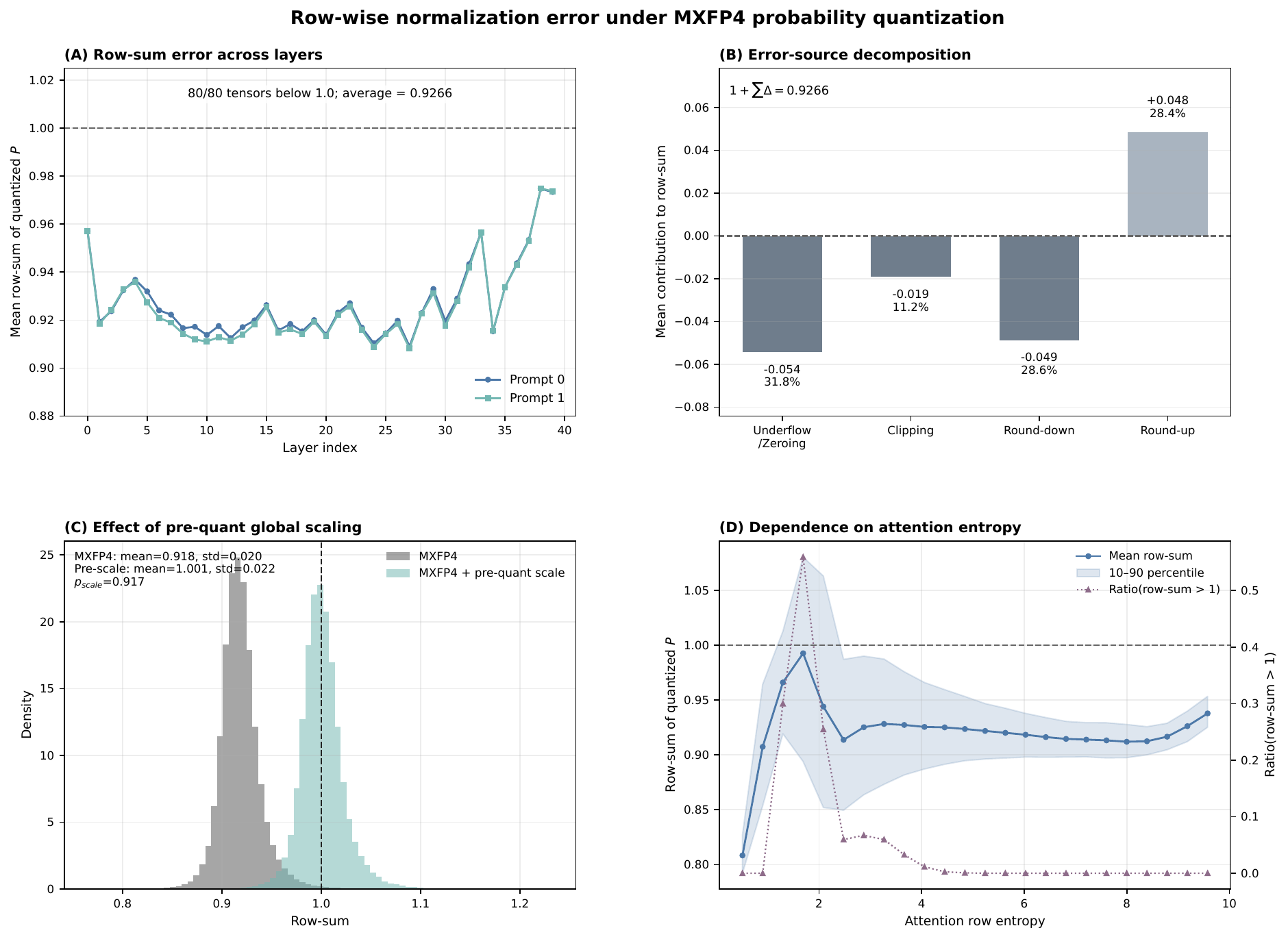}
\caption{\textbf{Row-sum diagnostics for MXFP4 probability quantization.} (A) Across 80 Wan2.2 attention tensors at denoising step~9, covering two prompts and 40 layers, the mean row sum of each quantized tensor is below one. (B) The row-sum error is decomposed into zeroing, saturation, downward rounding, and upward rounding. (C) Pre-quant global scaling can restore the average row sum of a representative tensor but does not enforce normalization for individual rows. (D) High-entropy rows are predominantly under-normalized, whereas low-entropy rows may be under- or over-normalized.}
\label{fig:pnq_rowsum}
\end{figure*}

Across the 80 tensors in Figure~\ref{fig:pnq_rowsum}(A), the average of the tensor-level mean row sums is 0.9266. This does not imply that every row has a sum below one; individual rows can be under- or over-normalized. It instead shows a consistent negative bias after aggregation across layers and prompts.

Figure~\ref{fig:pnq_rowsum}(B) separates the error by quantization outcome. Averaged across the tensors, zeroing contributes \(-0.0541\), saturation contributes \(-0.0190\), downward rounding contributes \(-0.0486\), and upward rounding contributes \(+0.0483\). The row-sum error is therefore not equivalent to a single global scaling factor.

Figure~\ref{fig:pnq_rowsum}(C) evaluates a pre-quant global scaling baseline, which applies a single scaling factor to the entire softmax tile before MXFP4 quantization. Although global scaling can restore the average row sum of a representative tensor, it substantially increases the number of rows with over-normalized attention weights. Correcting the mean therefore does not restore row-wise normalization.

Finally, Figure~\ref{fig:pnq_rowsum}(D) shows that row-sum behavior depends on attention entropy. Rows with entropy above 5 account for approximately 75\% of the representative tensor and are almost always under-normalized, while lower-entropy rows exhibit errors in both directions. PNQ avoids this row-dependent normalization error by using the same quantized exponential tile in the row-wise sum and output-accumulator updates.

\FloatBarrier

\subsection{Algorithmic Overhead and Kernel Integration}
\label{sec:hardware_efficiency}

MXAttention preserves the tiled dataflow of low-precision attention. UOS changes only the boundary used in shared-scale selection; it retains the same block-maximum reduction, quantization granularity, and tensor traversal as direct MXFP4 quantization. PNQ reuses the quantized exponential tile \(\widehat P\), already required by the low-precision output-accumulator update, to update the row-wise sum \(\ell\). It therefore requires neither an additional quantization operation nor another traversal of the attention tiles.

Moreover, rotation and quantization can be executed in parallel by leveraging orthogonal compute units (GEMM vs. non-GEMM units), eliminating resource contention. The fixed Hadamard transform is applied online only to \(Q\) and \(K\); the corresponding transformation outside this path is folded into the model weights. Because the rotation and block quantization operate locally on \(Q/K\) tiles, they can be fused into the \(Q/K\) preprocessing stage.

These modifications add no attention-matrix materialization or extra pass over the key/value sequence. They therefore introduce negligible algorithmic overhead and are compatible with a fused low-precision attention implementation.

\section{Conclusion}
\label{sec:conclusion}

This paper presents \textbf{MXAttention}, a data-free post-training quantization framework for MXFP4 attention in video diffusion models. We identify two numerical failure modes in direct MXFP4 attention: the clipping--underflow trade-off introduced by power-of-two block scaling and the mismatch between the row-wise sum and output-accumulator update caused by inconsistent quantization in the online-softmax loop. To address them, \textbf{Universal Optimal Scaling (UOS)} derives the closed-form boundary \(Q_{\max}=7.25\) as the distribution-independent minimizer of a global MXFP4 quantization-error objective, while \textbf{Pre-Normalization Quantization (PNQ)} uses the same quantized softmax exponential tiles for both the row-wise sum and output-accumulator updates, preserving row-wise normalization by construction.

Experiments on Wan2.2 and HunyuanVideo show that MXAttention recovers at least 95\% of the VBench Imaging Quality gap between direct OCP MXFP4 and FP16, substantially improves frame-level similarity, and maintains FP16-level generation quality with less than 0.01 absolute degradation on all reported VBench metrics. MXAttention also achieves results competitive with strong NVFP4-based baselines. UOS changes only the shared-scale selection, and PNQ reuses the quantized exponential tiles already required by the low-precision output path; neither requires an additional pass over the attention matrix. These results demonstrate that accurate MXFP4 attention can be achieved without calibration, quantization-aware training, or layer-wise parameter search.

\bibliographystyle{IEEEtran}
\bibliography{references}

\appendix

\section{Exact Density under Power-of-Two Scaling}
\label{app:uos_density}

We analyze nonzero blocks for which the required shared-scale exponent lies within the E8M0 exponent range. The cases \(M=0\) and exponents outside this range are handled by the MXFP4 conversion routine and are omitted because they do not affect the derivation of \(q\). Let \(M>0\) be the unscaled block maximum, let \(U=\log_2M\) have density \(f_U\), and define
\begin{equation}
e_q(M)=\left\lceil\log_2\left(\frac Mq\right)\right\rceil,
\qquad
X_q=\frac{M}{2^{e_q(M)}}.
\end{equation}
For \(x\in(q/2,q)\), every preimage of \(x\) has the form
\begin{equation}
M=2^kx,
\qquad
U=\log_2x+k,
\qquad
k\in\mathbb Z.
\end{equation}
Indeed,
\begin{equation}
e_q(2^kx)=\left\lceil k+\log_2\left(\frac xq\right)\right\rceil=k,
\end{equation}
because \(\log_2(x/q)\in(-1,0)\). The Jacobian of \(U=\log_2x+k\) is
\begin{equation}
\left|\frac{dU}{dx}\right|=\frac1{x\ln2}.
\end{equation}
Summing over all preimages gives
\begin{equation}
g_q(x)=\frac1{x\ln2}\sum_{k\in\mathbb Z}f_U(\log_2x+k),
\qquad x\in(q/2,q).
\label{eq:app_exact_density}
\end{equation}

Define the periodized log-domain density
\begin{equation}
H(t)=\sum_{k\in\mathbb Z}f_U(t+k).
\label{eq:app_periodized_density}
\end{equation}
The function \(H\) has period one: \(H(t+1)=H(t)\). Equation~\eqref{eq:app_exact_density} can therefore be written as
\begin{equation}
g_q(x)=g(x)=\frac{H(\log_2x)}{x\ln2}.
\label{eq:app_density_compact}
\end{equation}
Thus, \(q\) selects a factor-of-two interval of a fixed periodized density rather than changing the density inside that interval.

The density integrates to one over every such interval. Let \(a=\log_2q\). Then
\begin{align}
\int_{q/2}^{q}g(x)\,dx
&=\int_{a-1}^{a}H(t)\,dt\\
&=\int_0^1H(t)\,dt\\
&=\sum_{k\in\mathbb Z}\int_k^{k+1}f_U(u)\,du\\
&=1.
\end{align}

For almost every \(q>0\), periodicity gives
\begin{align}
g\left(\frac q2\right)
&=\frac{H(\log_2q-1)}{(q/2)\ln2}\\
&=2\frac{H(\log_2q)}{q\ln2}\\
&=2g(q).
\label{eq:app_boundary_ratio}
\end{align}
This proves Lemma~\ref{lemma:periodic_boundary_matching}. The result requires only the exact periodicity of \(H\), not a log-uniform, Gaussian, or small-perturbation assumption.

\section{Integrated Quantization Error of the E2M1 Grid}
\label{app:e2m1_integrals}

For nonnegative magnitudes, the finite E2M1 grid and its round-to-nearest decision boundaries are
\begin{align}
\mathcal G_{\mathrm{E2M1}}&=\{0,0.5,1,1.5,2,3,4,6\},\\
\mathcal T_{\mathrm{E2M1}}&=\{0.25,0.75,1.25,1.75,2.5,3.5,5.0\}.
\end{align}
The value \(6\) is the largest finite E2M1 value, and \(5\) is the decision boundary between \(4\) and \(6\). Hence, every normalized value above \(5\) projects to \(6\). The threshold \(7\) has a different role: it is the midpoint between \(6\) and the next value \(8\) on the unbounded E2M1 ladder. Values above \(7\) would round upward to \(8\) before finite-range saturation forces them to \(6\). We refer to this subset as the overflow-rounding region.

For an interval \([a,b]\) mapped to a quantization value \(r\),
\begin{equation}
\int_a^b(v-r)^2\,dv=\frac{(b-r)^3-(a-r)^3}{3}.
\label{eq:app_interval_integral}
\end{equation}
Applying Eq.~\eqref{eq:app_interval_integral} successively gives the cumulative errors in Table~\ref{tab:e2m1_integrals}.

\begin{table}[!t]
\centering
\caption{Exact cumulative quantization error \(E(x)\) at E2M1 decision boundaries and selected grid points.}
\label{tab:e2m1_integrals}
\renewcommand{\arraystretch}{1.25}
\resizebox{\columnwidth}{!}{
\begin{tabular}{ccccccccc}
\toprule
\(E(0)\) & \(E(0.25)\) & \(E(0.75)\) & \(E(1.25)\) & \(E(1.75)\) & \(E(2.5)\) & \(E(3.5)\) & \(E(5)\) & \(E(6)\)\\
\midrule
\(0\) & \(\frac1{192}\) & \(\frac1{64}\) & \(\frac5{192}\) & \(\frac7{192}\) & \(\frac1{12}\) & \(\frac16\) & \(\frac{13}{24}\) & \(\frac78\)\\
\bottomrule
\end{tabular}}
\end{table}

For \(q\ge6\), every value above \(5\) maps to \(6\), giving
\begin{equation}
E(q)=E(6)+\int_6^q(v-6)^2\,dv=\frac78+\frac13(q-6)^3.
\label{eq:app_upper_error}
\end{equation}
For \(6\le q\le7\), \(q/2\in[3,3.5]\), where values map to \(3\). Hence,
\begin{equation}
E\left(\frac q2\right)=\frac18+\frac13\left(\frac q2-3\right)^3.
\label{eq:app_lower_error_a}
\end{equation}
For \(7\le q\le8\), \(q/2\in[3.5,4]\), where values map to \(4\). Hence,
\begin{equation}
E\left(\frac q2\right)=\frac5{24}+\frac13\left(\frac q2-4\right)^3.
\label{eq:app_lower_error_b}
\end{equation}
Substitution into \(\Delta(q)=E(q)-8E(q/2)\) gives
\begin{equation}
\Delta(q)=
\begin{cases}
-\dfrac18, & 6\le q\le7,\\[6pt]
\dfrac{(4q-29)(4q-27)}8, & 7\le q\le8.
\end{cases}
\label{eq:app_relevant_delta}
\end{equation}

\section{Global Optimality of the UOS Boundary}
\label{app:uos_global_optimality}

Since \(D(x)=E(x)/x^3\),
\begin{equation}
D(q)-D\left(\frac q2\right)=\frac{E(q)-8E(q/2)}{q^3}=\frac{\Delta(q)}{q^3}.
\label{eq:app_distortion_difference}
\end{equation}

\subsection{Justification of Differentiation}
\label{app:differentiation}

Let \(F(x)=D(x)g(x)\). Since \(D\) is continuous on \((0,\infty)\) and \(g\) is locally integrable, \(F\in L^1_{\mathrm{loc}}(0,\infty)\). Define
\begin{equation}
G(t)=\int_1^tF(x)\,dx.
\end{equation}
Then \(G\) is absolutely continuous on every compact subinterval of \((0,\infty)\) and \(G'(t)=F(t)\) for almost every \(t\). Since
\begin{equation}
\mathcal J(q)=G(q)-G(q/2),
\end{equation}
\(\mathcal J\) is absolutely continuous and, for almost every \(q\),
\begin{equation}
\mathcal J'(q)=F(q)-\frac12F(q/2).
\end{equation}
Using Eq.~\eqref{eq:app_boundary_ratio},
\begin{align}
\mathcal J'(q)
&=D(q)g(q)-\frac12D\left(\frac q2\right)g\left(\frac q2\right)\\
&=g(q)\left[D(q)-D\left(\frac q2\right)\right]\\
&=\frac{g(q)}{q^3}\Delta(q).
\label{eq:app_objective_derivative}
\end{align}
Because \(g(q)\ge0\), the derivative sign is determined by \(\Delta(q)\).

Table~\ref{tab:global_delta} gives the exact piecewise form of \(\Delta(q)\) over \(q>0\).

\begin{table*}[!t]
\centering
\caption{Exact piecewise form and sign of \(\Delta(q)=E(q)-8E(q/2)\).}
\label{tab:global_delta}
\renewcommand{\arraystretch}{1.2}
\small
\begin{tabular}{lll}
\toprule
\textbf{Interval} & \(\boldsymbol{\Delta(q)}\) & \textbf{Sign}\\
\midrule
\(0<q\le\frac14\) & \(0\) & \(0\)\\
\(\frac14<q\le\frac12\) & \(-\dfrac{(4q-1)^2}{32}\) & \(<0\)\\
\(\frac12<q\le\frac34\) & \(\dfrac{16q^2-24q+7}{32}\) & \(<0\)\\
\(\frac34<q\le\frac54\) & \(-\dfrac1{16}\) & \(<0\)\\
\(\frac54<q\le\frac32\) & \(-\dfrac{16q^2-40q+27}{32}\) & \(<0\)\\
\(\frac32<q\le\frac74\) & \(\dfrac{(4q-9)(4q-5)}{32}\) & \(<0\)\\
\(\frac74<q\le7\) & \(-\dfrac18\) & \(<0\)\\
\(7<q\le10\) & \(\dfrac{(4q-29)(4q-27)}8\) & \(<0\) before \(29/4\), \(=0\) at \(29/4\), \(>0\) afterward\\
\(q>10\) & \(\dfrac{48q^2-864q+3983}{8}\) & \(>0\)\\
\bottomrule
\end{tabular}
\end{table*}

The quadratic in the final row has discriminant \(-18240<0\) and a positive leading coefficient, so it is strictly positive. Therefore,
\begin{equation}
\Delta(q)
\begin{cases}
=0, & 0<q\le\frac14,\\
<0, & \frac14<q<\frac{29}{4},\\
=0, & q=\frac{29}{4},\\
>0, & q>\frac{29}{4}.
\end{cases}
\label{eq:app_global_delta_sign}
\end{equation}
Combining Eqs.~\eqref{eq:app_objective_derivative} and~\eqref{eq:app_global_delta_sign}, \(\mathcal J(q)\) is nonincreasing before \(q=29/4\) and nondecreasing afterward. Hence \(q^\star=29/4\), equivalently \(Q_{\max}^\star=7.25\), is a global minimizer for every absolutely continuous law of \(\log_2M\) within the stated scale family.

If \(H\) is positive almost everywhere over one period, \(\mathcal J(q)\) is constant on \(0<q\le1/4\), strictly decreasing on \((1/4,29/4)\), and strictly increasing on \((29/4,\infty)\). Therefore, \(Q_{\max}^\star=29/4\) is the unique global minimizer.

\section{Wrapped-Distribution and Harmonic Interpretation}
\label{app:phase_uniformity}

Reducing a real-valued random variable modulo one yields a periodic density, commonly called a wrapped distribution~\cite{mardia2000directional,bell2024wrapped}. It can be represented either as a sum of translated densities or through Fourier coefficients. These classical results provide the mathematical tools used below; the contribution of UOS is the boundary-density consequence for power-of-two microscaling.

Let
\begin{equation}
Y_q=(U-\log_2q)\bmod1\in[0,1).
\end{equation}
For noninteger \(U-\log_2q\), which occurs almost surely under an absolutely continuous law,
\begin{equation}
X_q=q\,2^{Y_q-1}.
\label{eq:app_phase_mapping}
\end{equation}
At integer endpoints, the ceiling convention maps the value to \(q\); this measure-zero distinction does not affect the density or the quantization-error integrals.

If \(Y_q\) is uniformly distributed, the Jacobian gives
\begin{equation}
f_{X_q}(x)=\frac1{x\ln2},
\qquad x\in(q/2,q].
\label{eq:app_log_uniform_density}
\end{equation}
For the ceiling-based \(q=8\) interval \((4,8]\), which differs from the exact OCP interval \([4,8)\) only at measure-zero endpoints, this approximation gives
\begin{equation}
\Pr(X_8>7)=\int_7^8\frac1{x\ln2}\,dx=\log_2\frac87\approx19.27\%.
\label{eq:app_overflow_probability}
\end{equation}
This is a diagnostic estimate under the uniform wrapped-phase approximation, not an assumption required by Theorem~\ref{thm:uos_optimality}.

\subsection{Fourier Representation}

The Fourier coefficients of the periodized density \(H\) are
\begin{align}
c_n
&=\int_0^1H(t)e^{-i2\pi nt}\,dt\\
&=\int_{-\infty}^{\infty}f_U(u)e^{-i2\pi nu}\,du\\
&=\phi_U(-2\pi n),
\end{align}
where \(\phi_U(\omega)=\mathbb E[e^{i\omega U}]\). The zeroth coefficient is \(c_0=1\). Under standard Fourier-convergence conditions~\cite{stein2003fourier},
\begin{equation}
H(t)=1+\sum_{n\ne0}\phi_U(-2\pi n)e^{i2\pi nt}.
\label{eq:app_wrapped_fourier}
\end{equation}
For a general integrable density, the coefficient identity remains valid even when pointwise convergence of the Fourier series is not assumed.

Equation~\eqref{eq:app_wrapped_fourier} gives the harmonic interpretation of Lemma~\ref{lemma:periodic_boundary_matching}. For every integer \(n\),
\begin{equation}
e^{i2\pi n(\log_2q-1)}=e^{i2\pi n\log_2q}.
\end{equation}
Thus, the uniform term and every nonzero Fourier mode take the same value at the two scale boundaries. The harmonic components can change the objective value and the magnitude of its derivative, but they cannot reverse the derivative sign or shift the UOS minimizer.

If \(U\sim\mathcal N(\mu,\sigma^2)\), the Fourier series becomes
\begin{equation}
H(t)=1+2\sum_{n=1}^{\infty}e^{-2\pi^2n^2\sigma^2}\cos\left(2\pi n(t-\mu)\right).
\label{eq:app_wrapped_gaussian}
\end{equation}
The first-harmonic amplitude is
\begin{equation}
a_1(\sigma)=2e^{-2\pi^2\sigma^2}.
\label{eq:app_first_harmonic}
\end{equation}

\begin{table}[!t]
\centering
\caption{First-harmonic amplitude under a Gaussian log-domain prior. The values describe the leading Fourier component, not the complete pointwise deviation of the wrapped density.}
\label{tab:first_harmonic}
\renewcommand{\arraystretch}{1.15}
\begin{tabular}{ccc}
\toprule
\(\boldsymbol{\sigma}\) & \(\boldsymbol{\sigma^2}\) & \(\boldsymbol{a_1(\sigma)}\)\\
\midrule
0.10 & 0.01 & 1.6417\\
0.20 & 0.04 & 0.9081\\
0.30 & 0.09 & 0.3384\\
0.40 & 0.16 & 0.0850\\
0.50 & 0.25 & 0.0144\\
0.60 & 0.36 & 0.0016\\
1.00 & 1.00 & \(5.35\times10^{-9}\)\\
\bottomrule
\end{tabular}
\end{table}

Higher harmonics may remain non-negligible when \(\sigma^2\) is small, so \(a_1(\sigma)\) should not be interpreted as a bound on the total deviation from uniformity.

\section{Wrapped-Phase Convergence for Location--Scale Families}

Let \(U_\sigma=\mu+\sigma Z\), where \(Z\) has an absolutely continuous density. For every nonzero integer \(n\),
\begin{equation}
\left|\phi_{U_\sigma}(-2\pi n)\right|=\left|\phi_Z(-2\pi n\sigma)\right|.
\end{equation}
Since \(Z\) has an integrable density, the Riemann--Lebesgue lemma gives \(\lim_{|t|\to\infty}\phi_Z(t)=0\)~\cite{stein2003fourier}. Therefore, \(\phi_{U_\sigma}(-2\pi n)\to0\) for every \(n\ne0\). By the Fourier characterization of the uniform distribution modulo one~\cite{kuipers1974uniform},
\begin{equation}
U_\sigma\bmod1\xrightarrow{d}\mathcal U(0,1).
\end{equation}
Thus, broad absolutely continuous location--scale families approach a uniform wrapped phase. Without additional regularity assumptions, this weak convergence does not imply uniform convergence of the corresponding densities.

\section{Empirical Wrapped-Phase Diagnostics}
\label{app:wrapped_diagnostics}

The following figures visualize block-maximum distributions before and after logarithmic wrapping. Figure~\ref{fig:qkvp_distribution} shows empirical distributions from Wan2.2 attention tensors, including the wrapped log-domain phases used in our analysis. Real attention tensors can exhibit clear nonuniform harmonic structure; UOS does not require these harmonics to be small.

\begin{figure*}[!t]
\centering
\includegraphics[width=\textwidth]{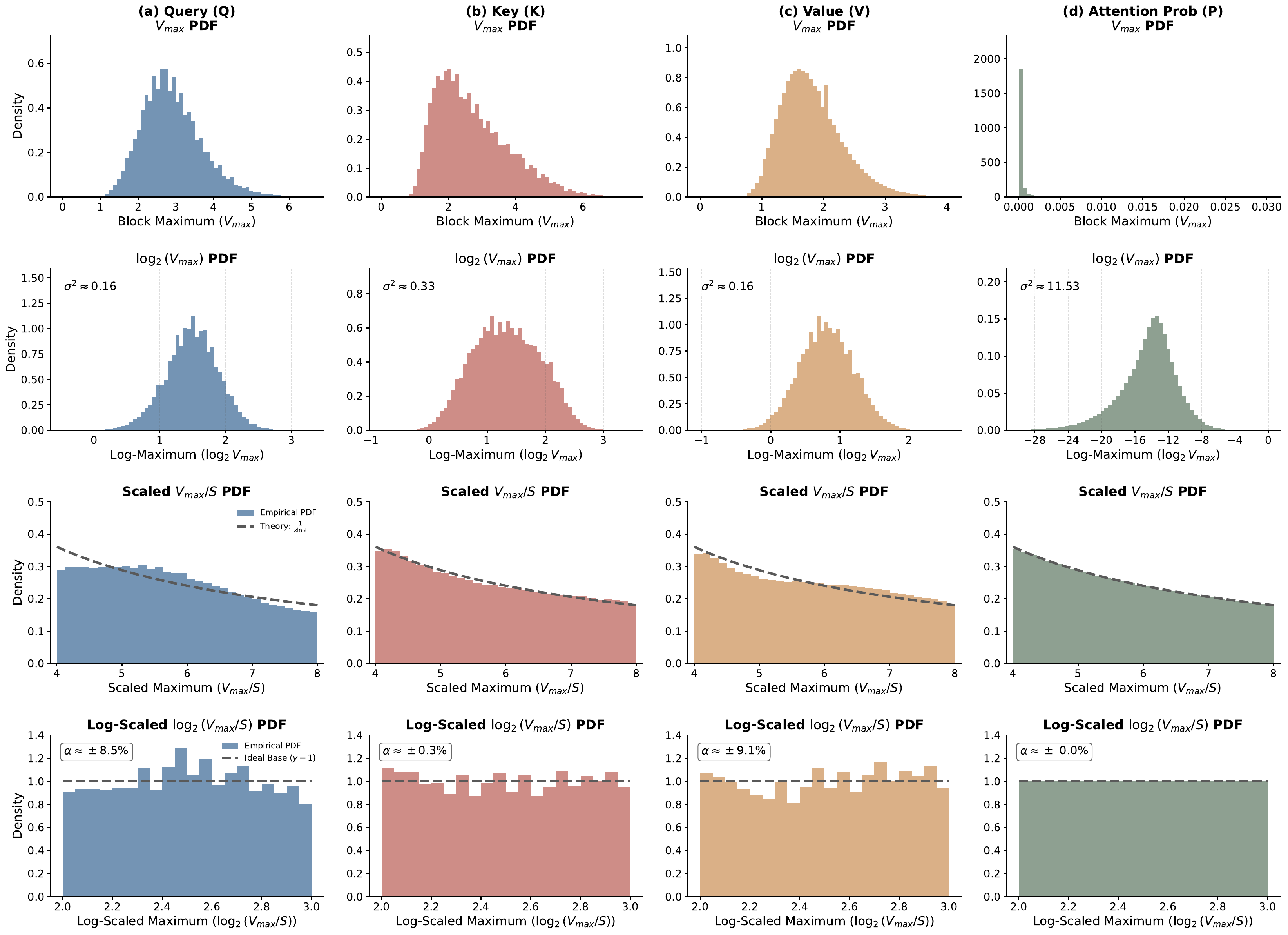}
\caption{\textbf{Empirical block-maximum distributions in Wan2.2.} The rows show the original block maxima, their logarithms, the wrapped log-domain phases, and the corresponding normalized linear-domain values for \(Q\), \(K\), \(V\), and the normalized attention probability matrix \(P\). The visible nonuniformity corresponds to distribution-specific harmonic components but does not change the UOS boundary condition. The \(P\) column is included only as a distributional diagnostic; PNQ quantizes the unnormalized softmax exponential tile.}
\label{fig:qkvp_distribution}
\end{figure*}

Figure~\ref{fig:rotation_impact} further examines how Hadamard rotation changes the empirical block-maximum distributions. Although the rotation modifies the log-domain harmonic structure, the periodic boundary relation used by UOS remains unchanged.

\begin{figure*}[!t]
\centering
\includegraphics[width=\textwidth]{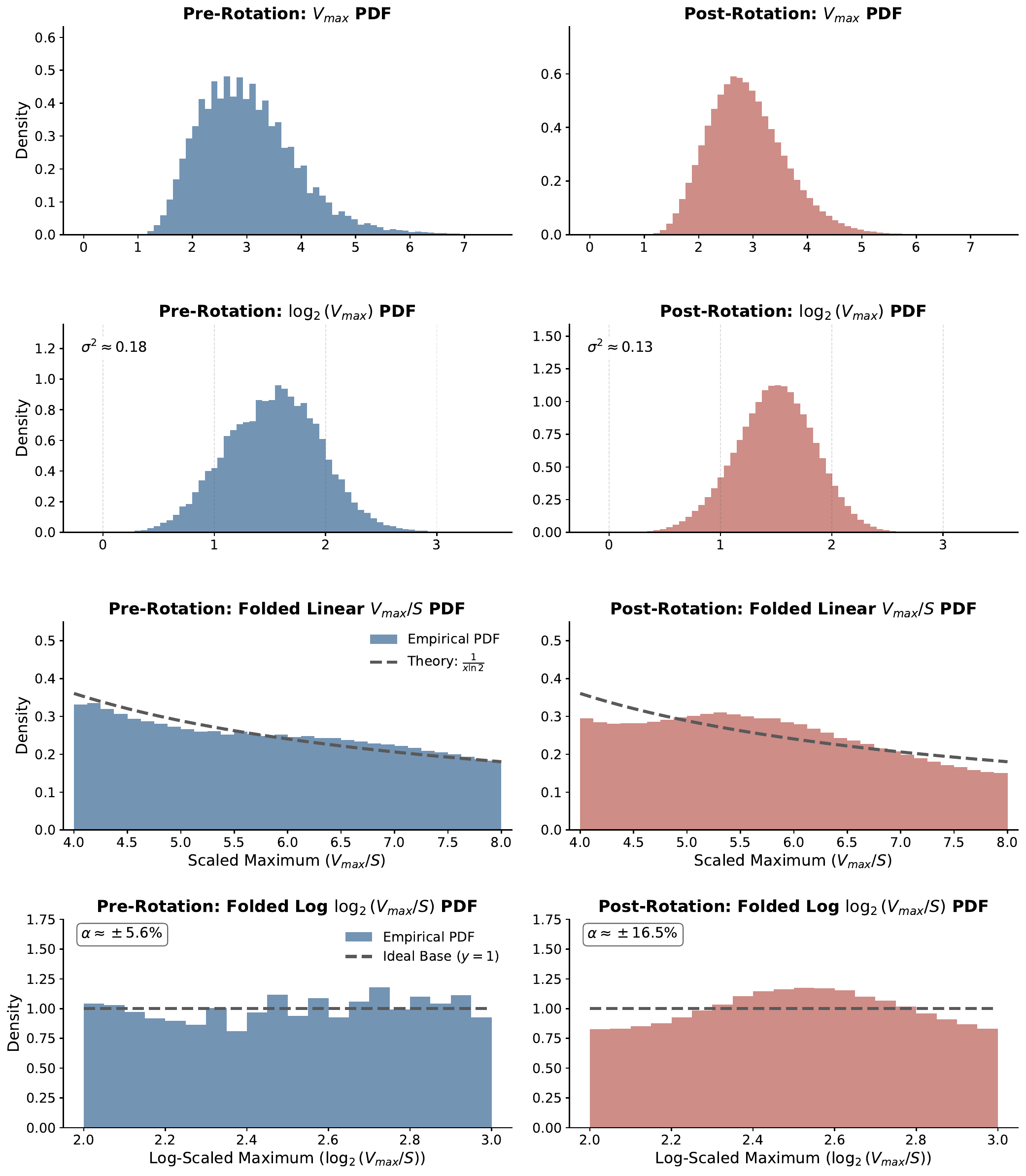}
\caption{\textbf{Effect of Hadamard rotation on block-maximum distributions.} Rotation changes the log-domain distribution and its harmonic content; the boundary density ratio remains unchanged.}
\label{fig:rotation_impact}
\end{figure*}

\section{Relation to OAS and the OCP Endpoint Convention}
\label{app:oas_relation}

At the common 32-element block granularity, the OAS scale-selection rule is equivalent to the ceiling-based boundary \(q=7\). Let
\begin{equation}
e_6=\left\lceil\log_2\left(\frac M6\right)\right\rceil,
\qquad
x_6=\frac{M}{2^{e_6}}\in(3,6].
\end{equation}
Then
\begin{equation}
e_7=e_6+\left\lceil\log_2\left(\frac{x_6}{7}\right)\right\rceil
=
\begin{cases}
e_6-1, & 3<x_6\le3.5,\\
e_6, & 3.5<x_6\le6.
\end{cases}
\end{equation}
Thus, \(q=7\) halves the TFS scale when \(x_6\in(3,3.5]\), mapping this interval to \((6,7]\), and otherwise leaves the scale unchanged. This is the OAS scale-selection rule at the same block granularity~\cite{oas}. The comparison does not include the smaller block size used in the full OAS configuration.

The experimental OCP baseline uses the exact floor-based conversion rule
\begin{equation}
e_{\mathrm{OCP}}=\lfloor\log_2M\rfloor-2
=\left\lfloor\log_2\left(\frac M8\right)\right\rfloor+1.
\label{eq:app_ocp_floor}
\end{equation}
The ceiling-based \(q=8\) rule used in the continuous boundary analysis is
\begin{equation}
e_8(M)=\left\lceil\log_2\left(\frac M8\right)\right\rceil.
\label{eq:app_q8_ceil}
\end{equation}
The two expressions agree whenever \(\log_2(M/8)\notin\mathbb Z\), since \(\lceil z\rceil=\lfloor z\rfloor+1\) for noninteger \(z\). They differ when \(M/8\) is an exact power of two. Under an absolutely continuous law, these endpoints have measure zero and do not affect the quantization-error integrals or Theorem~\ref{thm:uos_optimality}. All reported OCP experiments use the exact floor-based rule.

\section{Full Online-Softmax Form of PNQ}
\label{app:pnq_full_row}

Let \(T_c\) be the number of key/value tiles processed for query block \(i\). For each tile \(j\), define the row-wise rescaling factor from its current running maximum to the final running maximum as
\begin{equation}
\beta_i^{(j)}
=
\exp\!\left(m_i^{(j)}-m_i^{(T_c)}\right).
\label{eq:pnq_beta}
\end{equation}
Because
\begin{equation}
\alpha_i^{(r)}
=
\exp\!\left(m_i^{(r-1)}-m_i^{(r)}\right),
\end{equation}
the element-wise product of all subsequent rescaling factors telescopes:
\begin{equation}
\prod_{r=j+1}^{T_c}\alpha_i^{(r)}
=
\exp\!\left(m_i^{(j)}-m_i^{(T_c)}\right)
=
\beta_i^{(j)},
\label{eq:pnq_telescope}
\end{equation}
where the empty product for \(j=T_c\) is the all-ones vector.

\begin{lemma}[Effective Full-Row Form]
\label{lemma:pnq_effective_full_row}
After all key/value tiles have been processed, PNQ satisfies
\begin{equation}
\begin{aligned}
\ell_i^{(T_c)}
&=
\sum_{j=1}^{T_c}
\beta_i^{(j)}\odot
\left(\widehat P_{ij}\mathbf1_{B_c}\right),
\\
\widetilde O_i^{(T_c)}
&=
\sum_{j=1}^{T_c}
\operatorname{Diag}\!\left(\beta_i^{(j)}\right)
\widehat P_{ij}\widehat V_j.
\end{aligned}
\label{eq:pnq_unrolled}
\end{equation}
\end{lemma}

\begin{proof}
Unrolling Eq.~\eqref{eq:pnq_online_update}, the contribution of tile \(j\) is multiplied by every subsequent rescaling factor \(\alpha_i^{(r)}\), for \(r=j+1,\ldots,T_c\). Equation~\eqref{eq:pnq_telescope} reduces their element-wise product to \(\beta_i^{(j)}\), yielding Eq.~\eqref{eq:pnq_unrolled}.
\end{proof}

Define the effective quantized tile
\begin{equation}
\widehat P_{ij}^{\mathrm{eff}}
=
\operatorname{Diag}\!\left(\beta_i^{(j)}\right)\widehat P_{ij},
\end{equation}
and the conceptual full-row matrix
\begin{equation}
\widehat P_i^{\mathrm{eff}}
=
\left[
\widehat P_{i1}^{\mathrm{eff}},
\ldots,
\widehat P_{iT_c}^{\mathrm{eff}}
\right].
\end{equation}
Let \(\widehat V\) denote the corresponding vertical concatenation of \(\widehat V_1,\ldots,\widehat V_{T_c}\). Lemma~\ref{lemma:pnq_effective_full_row} then gives
\begin{equation}
\ell_i^{(T_c)}
=
\widehat P_i^{\mathrm{eff}}\mathbf1,
\qquad
\widetilde O_i^{(T_c)}
=
\widehat P_i^{\mathrm{eff}}\widehat V.
\end{equation}
The matrix \(\widehat P_i^{\mathrm{eff}}\) is an analytical representation of the online recurrence and is never materialized by the kernel.

\section{Full-Row Mismatch under Direct Placement}
\label{app:pnq_direct_full_row}

For the direct placement, define
\begin{equation}
\widetilde P_{ij}^{\mathrm{eff}}
=
\operatorname{Diag}\!\left(\beta_i^{(j)}\right)\widetilde P_{ij},
\qquad
\widehat P_{ij}^{\mathrm{eff}}
=
\operatorname{Diag}\!\left(\beta_i^{(j)}\right)\widehat P_{ij},
\end{equation}
and concatenate the tiles across the key sequence:
\begin{equation}
\widetilde P_i^{\mathrm{eff}}
=
\left[
\widetilde P_{i1}^{\mathrm{eff}},
\ldots,
\widetilde P_{iT_c}^{\mathrm{eff}}
\right],
\qquad
\widehat P_i^{\mathrm{eff}}
=
\left[
\widehat P_{i1}^{\mathrm{eff}},
\ldots,
\widehat P_{iT_c}^{\mathrm{eff}}
\right].
\end{equation}
The final direct-placement states are
\begin{equation}
\ell_i^{\mathrm{direct}}
=
\widetilde P_i^{\mathrm{eff}}\mathbf1,
\qquad
\widetilde O_i^{\mathrm{direct}}
=
\widehat P_i^{\mathrm{eff}}\widehat V.
\end{equation}
Hence, the final weights induced by the direct placement are
\begin{equation}
\widehat P_i^{\mathrm{direct}}
=
\operatorname{Diag}\!\left(
\widetilde P_i^{\mathrm{eff}}\mathbf1
\right)^{-1}
\widehat P_i^{\mathrm{eff}},
\end{equation}
with row sums
\begin{equation}
\widehat P_i^{\mathrm{direct}}\mathbf1
=
\left(
\widehat P_i^{\mathrm{eff}}\mathbf1
\right)
\oslash
\left(
\widetilde P_i^{\mathrm{eff}}\mathbf1
\right).
\label{eq:pnq_direct_full_rowsum}
\end{equation}
Here, \(\oslash\) denotes element-wise division between the two row-wise sum vectors. The running-max rescaling is the same in both states, but the exponential tiles are not; the final row sum therefore need not equal one.

\section{Positive PNQ Normalization Factor}
\label{app:pnq_positive_normalizer}

Fix any row \(r\) of query block \(i\). All attention rows considered in this work have a nonempty key set. Let tile \(j^\star\) contain an entry \(c^\star\) attaining the final maximum of row \(r\):
\begin{equation}
S_{ij^\star}[r,c^\star]
=
m_i^{(T_c)}[r].
\end{equation}
When tile \(j^\star\) is processed, the running maximum for this row reaches its final value:
\begin{equation}
m_i^{(j^\star)}[r]
=
m_i^{(T_c)}[r].
\end{equation}
Therefore,
\begin{equation}
\widetilde P_{ij^\star}[r,c^\star]
=
\exp\!\left(
S_{ij^\star}[r,c^\star]
-
m_i^{(j^\star)}[r]
\right)
=
1,
\end{equation}
and
\begin{equation}
\beta_i^{(j^\star)}[r]=1.
\end{equation}

Every entry of \(\widetilde P_{ij^\star}\) lies in \((0,1]\), and the MXFP4 block containing \(\widetilde P_{ij^\star}[r,c^\star]\) therefore has maximum magnitude one. Under the UOS boundary \(q^\star=7.25\),
\begin{equation}
e_{q^\star}(1)
=
\left\lceil
\log_2\left(\frac{1}{7.25}\right)
\right\rceil
=
-2.
\end{equation}
The shared scale is \(2^{-2}=1/4\), so the normalized value is
\begin{equation}
\frac{1}{2^{-2}}=4,
\end{equation}
which is exactly representable in E2M1. Hence,
\begin{equation}
\mathcal Q_\star(1)=1.
\end{equation}
Since this entry also has \(\beta_i^{(j^\star)}[r]=1\), it contributes exactly one to the final row-wise sum. Therefore,
\begin{equation}
\ell_i^{(T_c)}[r]\ge1>0.
\end{equation}

The normalization guarantee concerns the weights induced by the PNQ recurrence in exact arithmetic. It does not assert that the quantized weights equal the full-precision softmax weights. In an implementation, finite-precision accumulation and the final division may introduce ordinary floating-point rounding. PNQ removes the additional structural mismatch caused by updating \(\ell_i\) with \(\widetilde P_{ij}\) while updating \(\widetilde O_i\) with \(\widehat P_{ij}\).

\end{document}